\documentclass[10pt,journal,compsoc]{IEEEtran}

\usepackage{graphicx}
\usepackage{amsfonts}
\usepackage{amsmath}
\usepackage{amssymb}
\usepackage{diagbox}
\usepackage{multirow}
\usepackage{graphics}
\usepackage{epstopdf}
\usepackage{tabularx}
\usepackage{epsfig}
\usepackage{color,subfigure}
\usepackage[T1]{fontenc}    
\usepackage{mathrsfs}
\usepackage{bm}
\usepackage{mathdots}
\usepackage[american]{babel}
\usepackage{microtype} 
\usepackage{array}
\usepackage{setspace}
\usepackage{threeparttable}
\usepackage{url}
\usepackage[colorlinks,linkcolor=red]{hyperref}
\usepackage{booktabs}
\usepackage{multirow}
\usepackage{bbding}
\usepackage{pifont}
\usepackage{colortbl}
\usepackage{cleveref}
\crefname{figure}{Fig}{Figs}
\Crefname{figure}{Fig}{Figs}
\crefname{section}{Sec}{Secs}
\Crefname{section}{Sec}{Secs}
\crefname{table}{Tab}{Tabs}
\Crefname{table}{Tab}{Tabs}
\crefname{equation}{Eq}{Eqs}
\Crefname{equation}{Eq}{Eqs}
\usepackage{algorithm,algorithmicx}
\usepackage{algpseudocode}
\definecolor{mydarkblue}{rgb}{0,0.08,0.45}
\usepackage{xcolor}
\newcolumntype{C}[1]{>{\centering\arraybackslash}m{#1}}
\usepackage{amsthm}
\usepackage{cancel}
\usepackage{dsfont}
\newtheorem{theorem}{Theorem}[section]

\ifCLASSOPTIONcompsoc

  \usepackage[nocompress]{cite}
\else
  \usepackage{cite}
\fi

\ifCLASSINFOpdf
\else
\fi

\begin{document}

\title{UniNDM: A Unified Noise-driven Detection and Mitigation Framework Against Sexual Content in Text-to-Image Generation}

\author{Yao Huang$^{*}$, Yitong Sun$^{*}$, Huanran Chen, Ruochen Zhang, Shouwei Ruan, Ranjie Duan, Maoxun Yuan, Yinpeng Dong, Hui Xue, Xiaochun Cao, \IEEEmembership{Senior Member, ~IEEE}, Xingxing Wei$^\ddagger$, \IEEEmembership{Member,~IEEE}
\IEEEcompsocitemizethanks{
\IEEEcompsocthanksitem Yao Huang, Yitong Sun, Ruochen Zhang, Shouwei Ruan, Maoxun Yuan and Xingxing Wei are with the Institute of Artificial Intelligence, State Key Laboratory of Virtual Reality Technology and Systems, Beihang University, Beijing 100191, China.
(E-mail: \{y\_huang, yt\_sun, xxwei\}@buaa.edu.cn)
\IEEEcompsocthanksitem Huanran Chen and Yinpeng Dong are with the College of Artificial Intelligence, Tsinghua University, Beijing 100083, China. 
\IEEEcompsocthanksitem Ranjie Duan and Hui Xue are with the Security Department, Alibaba Group, Hangzhou 310056, China. 
\IEEEcompsocthanksitem Xiaochun Cao is with the School of Cyber Science and Technology, Shenzhen Campus of Sun Yat-Sen University, Shenzhen 518107, China.
\IEEEcompsocthanksitem Yao Huang and Yitong Sun have equal contributions. ($*$).
\IEEEcompsocthanksitem Xingxing Wei is the corresponding author. ($\ddagger$)
\IEEEcompsocthanksitem This work was supported in part by the Project of the National Natural Science Foundation of China under Grants 62576020, 62441619, U2541229 and in part by the Fundamental Research Funds for the Central Universities.}
}


\IEEEtitleabstractindextext{%
\begin{abstract}
Despite the impressive generative capabilities of text-to-image (T2I) diffusion models, they remain vulnerable to implicit sexual prompts, where subtle cues disguised as benign terms or adversarial tokens unexpectedly generate the inappropriate content due to model biases or latent correlations in training data. Existing safety mechanisms face fundamental limitations: detection methods primarily identify explicit content and fail to capture implicit malicious intent, while mitigation approaches rely on static negative prompts inadequate for diverse implicit scenarios. To address these challenges, we propose \textbf{UniNDM}, a unified noise-driven framework that  rethinks safety mechanisms through the lens of noise dynamics in diffusion processes. Our key insight is that early-stage predicted noise exhibits inherent separability between normal and sexually explicit content, which we theoretically  demonstrates quadratically increasing semantic concentration with timestep. Leveraging this property, we develop a lightweight noise-based detector achieving superior accuracy with virtually no computational overhead. For mitigation, we introduce noise-enhanced adaptive negative guidance: dynamically generating context-specific negative prompts via large language models to handle diverse implicit content, while optimizing initial noise by suppressing attention concentration on explicit tokens to provide comprehensive protection. Besides the U-Net-based diffusion models, we further extend our framework to emerging Diffusion Transformer architectures through region-constrained semantic guidance tailored for their unified multimodal attention. Comprehensive experiments across U-Net models (SDv1.4, v1.5, v2.1, XL) and DiT models (SDv3) on both natural and adversarial datasets demonstrate substantial improvements over state-of-the-art methods, including SLD, UCE, Safree, etc. Our code is publicly available at \url{https://github.com/Aries-iai/UniNDM}.
\end{abstract}

}

\maketitle

\IEEEdisplaynontitleabstractindextext

\IEEEpeerreviewmaketitle

\IEEEraisesectionheading{\section{Introduction}\label{sec:introduction}}
\IEEEPARstart{S}{ignificant} advances in diffusion models~\cite{nichol2021glide, rombach2022high, Midjourney} have revolutionized text-to-image (T2I) synthesis, enabling the generation of highly photorealistic images from textual descriptions with unprecedented fidelity. This technological paradigm has facilitated rapid deployment across diverse application domains, including digital content creation~\cite{wang2024diffusion}, commercial visualization systems~\cite{saharia2022photorealistic}, and biomedical imaging synthesis~\cite{chambon2022roentgan, aversa2023diffinfinite}. Despite the transformative potential these models demonstrate for creative and scientific applications, their sophisticated generative architectures concurrently introduce significant security vulnerabilities. Specifically, adversarial exploitation of these systems can facilitate the synthesis of harmful content, particularly non-consensual explicit material~\cite{schramowski2023safe, qu2023unsafe, yang2024sneakyprompt}, including deepfake imagery targeting specific individuals or potentially illegal content such as child sexual abuse material, thereby precipitating critical ethical, legal, and societal challenges.

To mitigate these ethical challenges posed by T2I diffusion models, extensive research efforts have been devoted to developing safety mechanisms. These efforts can be broadly categorized into two main approaches: Model-intrinsic methods and Model-extrinsic methods. Model-intrinsic methods generally modify the model's internal parameters through techniques such as fine-tuning CLIP weights~\cite{poppi2024safe}, concept unlearning~\cite{gandikota2023erasing, huang2024receler}, and parameter editing~\cite{gandikota2024unified} to suppress the generation of explicit content. While effective in mitigating known undesirable outputs, these methods often suffer from a significant trade-off as they degrade the model's overall generation performance on benign tasks.
\begin{figure*}
    \centering
    \includegraphics[width=0.95\linewidth]{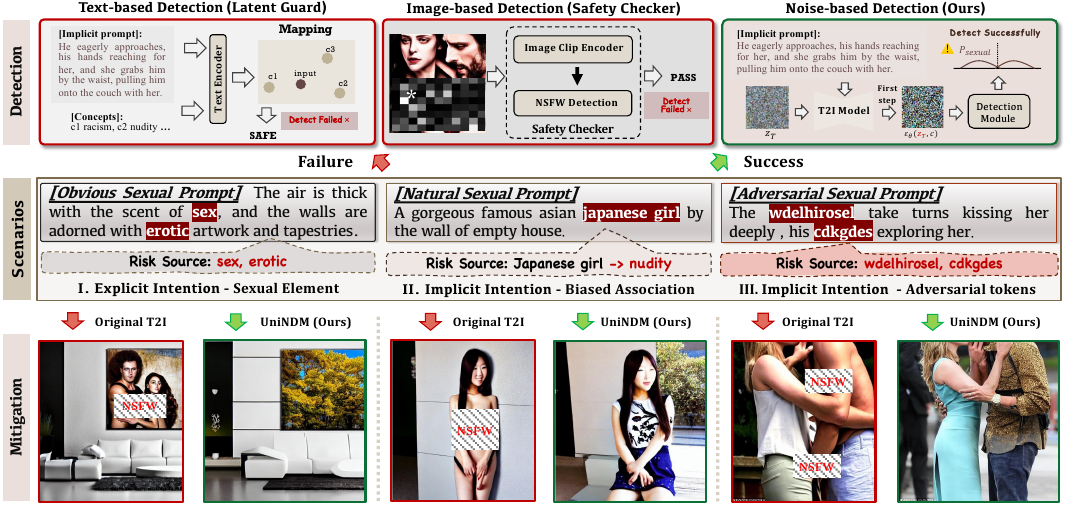}
    \caption{Middle: Pornographic scenarios can be categorized into explicit and implicit cases, with the latter further subdivided into two distinct types: natural and adversarial. Top: Detection challenges in existing methods. Text-based detection fails to capture implicit sexual intent due to reliance on prompt encoding and harmful concept comparison. Image-based methods, which operate on fully generated images in CLIP space, are hindered by computational inefficiency and delayed intervention. In contrast, our UniNDM leverages early-stage predicted noise, achieving superior efficiency and precision in detecting harmful content before full image generation. Bottom: Illustration of UniNDM's successful applications across various scenarios.}
    \label{fig:cover}
\end{figure*}

In contrast, Model-extrinsic methods focus on detection and mitigation strategies that preserve the model's general capabilities without internal modifications. These approaches employ external safeguards such as plug-and-play filters to detect inappropriate outputs via textual cues~\cite{liu2024latent} or generated imagery~\cite{safetychecker}, or steer generation in safer directions by manipulating prompt embeddings away from harmful subspaces~\cite{yoon2025safree} or applying negative guidance~\cite{rombach2022high}. However, these methods still struggle with implicit malicious intent, subtle sexual suggestions that arise from latent biases in training data and adversarially crafted inputs~\cite{tsai2024ring, chin2024prompting4debugging, yang2024sneakyprompt, zhang2024generate}. As illustrated in \Cref{fig:cover}, seemingly benign phrases like ``Japanese girl in a bedroom'' may trigger inappropriate content due to latent correlations, while optimized adversarial tokens (e.g., ``wdehirosel'', ``cdkgdes'') can subtly manipulate generation without triggering conventional filters, highlighting a critical unresolved gap in current safety mechanisms.

This paper aims to ensure safer text-to-image generation by inheriting the detection-and-mitigation framework's advantage of preserving model performance while crucially addressing the challenge of implicit malicious intention. Two fundamental questions emerge: (1) How can we improve early detection of implicit malicious intent? Existing text-based detectors focus on explicit keywords but fail to capture subtle intent hidden in benign prompts, while image-based methods require complete generation before assessment, introducing significant computational overhead. We need more efficient techniques that identify potential harm earlier in the generation pipeline. (2) How can we effectively mitigate implicit malicious intent during generation? Existing methods rely on static negative prompts (e.g., ``nudity'') that fail to handle diverse implicit malicious content from complex prompts. Moreover, simple negative guidance alone proves insufficient to prevent significant inappropriate outputs. We need a dynamic, context-aware mechanism that adapts to various implicit sexual prompts.

To address these challenges, we propose \textbf{UniNDM}, a \textbf{\underline{Uni}}fied \underline{N}oise-driven \textbf{\underline{D}}etection and \textbf{\underline{M}}itigation framework that tackles implicit malicious intention in text-to-image generation. For the first challenge, we discover that early-stage predicted noise exhibits significant separability between normal and sexually explicit content (as visualized in \Cref{fig:noise}). This stems from a key insight: while images contain more total information, early-stage noise captures more semantically concentrated information by filtering out irrelevant details and emphasizing the prompt's most prominent features. We formalize this through theoretical analysis (Theorem~\ref{thm:semantic_superiority}), proving that semantic concentration increases quadratically with timestep. Leveraging this property, we train a lightweight classifier on first-step predicted noise, achieving superior detection accuracy with virtually no additional computational cost.

For the second challenge, we develop a noise-enhanced adaptive negative guidance mechanism. To address the limitation of static negative prompts, we leverage large language models (LLMs) to dynamically generate context-specific negative prompts through linguistic analysis, enabling our approach to handle diverse implicit sexual content arising from complex and varied prompts. However, adaptive negative guidance alone may remain insufficient. Thus, we introduce a complementary enhancement through initial noise optimization, which is inspired by recent findings~\cite{guo2024initno, qi2024not, xu2025good, ban2024crystal} that different initial noise samples vary substantially in their generation. Specifically, by suppressing attention concentration on explicit tokens in early denoising steps through cross-attention analysis, we obtain safer initial noise that inherently reduces inappropriate content generation. 

While the above detection and mitigation mechanisms are effective for traditional U-Net-based diffusion models, the landscape of text-to-image generation has been evolving rapidly. The recent emergence of Diffusion Transformers (DiT), exemplified by Stable Diffusion 3's Multimodal Diffusion Transformer (MM-DiT), represents a paradigm shift in generative modeling~\cite{esser2024scaling}. Unlike U-Net architectures that rely on cross-attention for text conditioning, DiT employs a unified multimodal attention mechanism where image and text tokens co-evolve symmetrically within a shared semantic space. This architectural difference creates critical migration challenges: directly applying U-Net's negative guidance strategy (as in \Cref{eq:neg_guidance}) to DiT 
either proves largely ineffective or severely disrupts generation quality due to fundamentally different modality interactions between the two architectures. To address this, we first validate our noise-based detection's architecture-agnostic effectiveness on DiT models. For mitigation, we exploit DiT's semantic decoupling capability to develop region-constrained semantic guidance, which surgically intervenes only in risk-localized regions identified through attention analysis while preserving semantic integrity elsewhere. Combined with DiT-adapted initial noise optimization, this provides comprehensive protection for next-generation diffusion models.

In summary, the contributions of this paper are as follows:
\begin{itemize}
    \item We introduce UniNDM, the first unified noise-driven detection and mitigation framework that ensures safer image generation while preserving general model capabilities across both U-Net-based and DiT-based architectures. Leveraging the distinct attention mechanisms of each architecture, UniNDM proposes architecture-specific mitigation variants: cross-attention-based optimization for U-Net and region-constrained semantic guidance for DiT.
    
    \item We uncover two critical insights into the role of noise in safe text-to-image generation: the separability of early-stage predicted noise, which exhibits distinct patterns between harmful and benign content, enabling efficient and accurate detection with minimal computational overhead; and the significant impact of initial noise on sexual content generation, revealing that initial noise critically governs the manifestation intensity of explicit elements, leading to our noise-enhanced adaptive negative guidance mechanism.

    \item We conduct extensive experiments on both natural implicit and adversarial datasets for sexual content detection and mitigation, covering both U-Net-based (Stable Diffusion v1.4, v1.5, v2.1, XL) and DiT-based (Stable Diffusion v3) models. Our results demonstrate superior effectiveness compared to state-of-the-art methods including SLD~\cite{schramowski2023safe}, UCE~\cite{gandikota2023erasing}, and Safree~\cite{yoon2025safree}, particularly for challenging implicit sexual prompts and adversarial attacks.
\end{itemize}

This journal paper is an extended version of our conference paper~\cite{10.1145/3746027.3755326} published at ACMMM. Compared with the conference version, we have made substantial improvements and extensions in the following aspects: (1) We provide rigorous theoretical analysis to support our noise-based detection approach, including Theorem~\ref{thm:semantic_superiority} on the semantic superiority of early-stage noise prediction. This theoretical foundation, along with its formal proof, is detailed in \Cref{sec:noise_detection} and Appendix. (2) We extend our framework to support Diffusion Transformer (DiT) architectures, particularly Stable Diffusion 3's MM-DiT blocks. This extension includes a novel region-constrained semantic guidance mechanism tailored for DiT's unified multimodal attention, along with DiT-adapted initial noise optimization strategies. The technical details are introduced in \Cref{sec:noise_mitigation_dit}, and the corresponding experiments are presented in \Cref{sec:experiment}. (3) We significantly expand the experimental evaluation to include comprehensive comparisons with additional state-of-the-art methods, ablation studies on key components, and evaluations on both U-Net and DiT architectures across multiple challenging scenarios. (4) We have substantially revised the manuscript to improve clarity and presentation, including restructured sections, enhanced visualizations, and more detailed technical descriptions throughout \Cref{sec:method} and \Cref{sec:experiment}. These modifications substantially strengthen the paper's technical contribution and improve its overall quality.

This paper is organized as follows: \Cref{sec:related} reviews related work on safe text-to-image generation. \Cref{sec:method} presents our UniNDM framework in detail, including noise-based detection (\Cref{sec:noise_detection}), noise-enhanced adaptive mitigation for U-Net (\Cref{sec:noise_mitigation}), and the DiT extension (\Cref{sec:noise_mitigation_dit}). \Cref{sec:experiment} provides comprehensive experimental results and analysis. \Cref{sec:conclusion} concludes the paper with discussion and future directions.

\begin{figure*}[!t]
    \centering
    \includegraphics[width=0.98\linewidth]{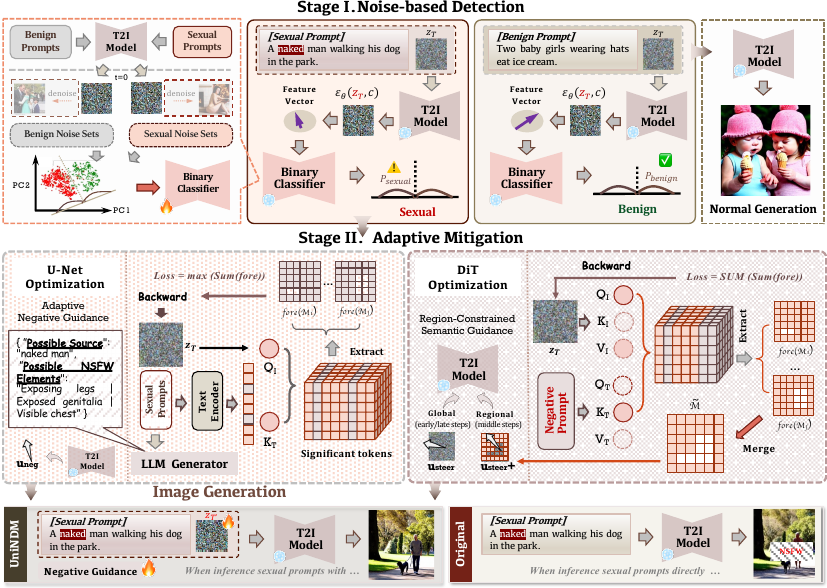}
    \caption{Overview of UniNDM. \textit{Stage I}: Noise-based Detection utilizes predicted noise separability to achieve classification. \textit{Stage II}: When sexual prompts are detected, adaptive mitigation begins by optimizing the initial noise through suppressing significant foreground regions in the attention maps. For U-Net architectures, this is followed by adaptive negative guidance with LLM-generated negative prompts tailored to the input. For DiT architectures, we employ region-constrained semantic guidance that selectively steers risk-localized regions identified through attention analysis, manipulating global and regional steering during different steps to neutralize inappropriate content while preserving non-risk semantics.}
\label{fig:framework}
\end{figure*}

\section{Related Works}
\label{sec:related}

\subsection{Text-to-Image Models}
Text-to-image synthesis has evolved significantly through diffusion models. Pioneering works like OpenAI’s DALL-E series~\cite{ramesh2021zero} demonstrate early potential, while proprietary tools like Midjourney~\cite{Midjourney} showcase commercial viability. Among them, Stability AI’s open-source Stable Diffusion series~\cite{rombach2022high} has driven widespread adoption and innovation.
Specifically, the foundational Stable Diffusion 1.x series~\cite{rombach2022high} introduces the Latent Diffusion Model (LDM), performing diffusion in a compressed VAE latent space with less computational cost; Stable Diffusion 2.1~\cite{rombach2022high} adopts OpenCLIP ViT-H/14~\cite{radford2021learning} for enhanced text understanding and also enables higher resolution generation; SDXL~\cite{podell2023sdxl} further adds a latent-based refiner, forming a two-stage pipeline with scaled U-Net and dual text encoders~\cite{radford2021learning} (CLIP ViT-L/14 and OpenCLIP ViT-bigG/14), significantly improving detail and realism.
Recently, DiT-based models have surpassed U-Net for T2I generation. For instance, Stable Diffusion 3~\cite{esser2024scaling} employs a multimodal DiT architecture with one more text encoder (T5-XXL~\cite{raffel2020exploring}) and flow matching, enhancing text comprehension and legible text rendering. In this paper, we focus on defending against explicit generation arising from diffusion models with both U-Net and DiT architectures.

\subsection{Ethical Risks with Text-to-Image Generation}
As text-to-image generation models advance, several ethical threats~\cite{zhang2024multitrust, xu2025mmdt, huang2025perception, d2025gradbias,huang2025t2i,peng2025unified} also emerge, particularly regarding the generation of sexual content.  To systematically assess this, Schramowski \textit{et al.} propose the I2P dataset~\cite{schramowski2023safe}, a collection of malicious prompts designed to evaluate the generation of inappropriate imagery. Their work reveals that open-source latent diffusion models, such as Stable Diffusion~\cite{rombach2022high}, continue to struggle with ensuring safe content generation. Among these, sexual content, which arises from implicit associations and underlying concepts rather than explicit statements, represents one of the most significant threats. Beyond this, some other studies have demonstrated that diffusion models are also vulnerable to sexual content caused by implicit adversarial manipulation.
For example, Prompting4Debugging~\cite{chin2024prompting4debugging} and Ring-a-bell~\cite{tsai2024ring} employ prompt engineering techniques to generate seemingly benign inputs but could lead to harmful outputs, akin to jailbreaks in LLMs~\cite{huang2025breaking,zeng2024johnny}. Similarly, SneakyPrompt~\cite{yang2024sneakyprompt} uses reinforcement learning to discover adversarial prompts that bypass safety filters while preserving harmful semantics. MMA-Diffusion~\cite{yang2024mma} further exploits both textual and visual inputs to evade the model’s safeguards. These studies underscore the pressing need for more robust countermeasures to address the risks posed by such implicit sexual prompts. 

\subsection{Defense Against Malicious Generation}
Research on defending against harmful content generation has primarily followed two paradigms: model-intrinsic and model-extrinsic approaches.
Model-intrinsic methods directly modify model parameters to suppress undesirable generation. Concept unlearning techniques like ESD~\cite{gandikota2023erasing} and Receler~\cite{huang2024receler} remove specific concepts by fine-tuning models to align predicted noises with negatively guided distributions or steer outputs toward neutral targets. Model-editing approaches such as UCE~\cite{gandikota2024unified} and RECE~\cite{gong2024reliable} offer more targeted interventions by modifying cross-attention weights. Safe-CLIP~\cite{poppi2024safe} fine-tunes text encoder weights to reduce harmful sensitivity. Recent works~\cite{ruan2025towards, liu2024safetydpo} adapt safety alignment from LLMs~\cite{zhang2025stair, zhang2025realsafe} through DPO-based training with safety constraints.

Model-extrinsic methods intervene during inference without parameter modification. Detection-based approaches like Latent Guard~\cite{liu2024latent} and safety checker~\cite{safetychecker} identify inappropriate content in the latent space or generated images. Mitigation techniques focus on steering generation: SLD~\cite{schramowski2023safe} removes inappropriate concepts by subtracting guidance from unsafe prompts, while Safree~\cite{yoon2025safree} projects prompt embeddings away from learned harmful subspaces. 

Despite these advances, existing methods face notable limitations: model-intrinsic approaches often compromise generation quality on benign inputs, while model-extrinsic methods primarily address explicit content and remain vulnerable to implicit malicious cues. Additionally, most techniques are tied to U-Net's cross-attention mechanisms and cannot directly transfer to DiT architectures. To address these limitations, our UniNDM framework aims to effectively detect and mitigate implicit sexual prompts while preserving generation quality on benign inputs, with unified applicability across both U-Net-based and DiT-based architectures.

\begin{figure*}[!t]
    \centering
    \includegraphics[width=0.97\linewidth]{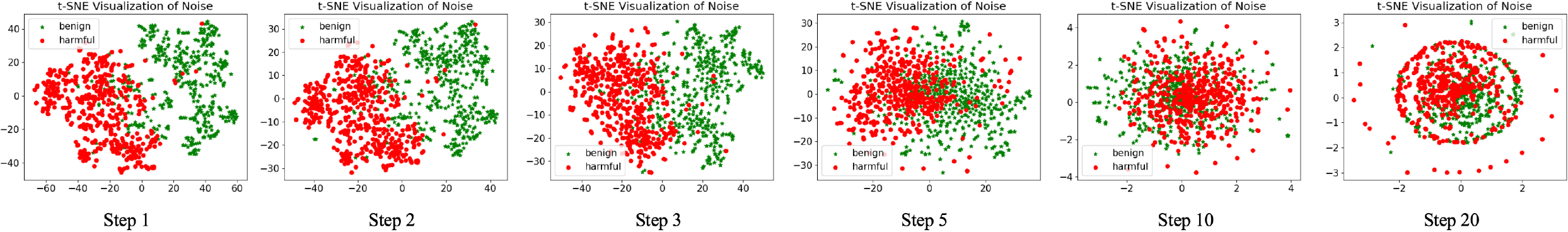}
        \caption{Visualized separability of predicted noises at different timesteps for benign and sexual generations using t-SNE.} 
    \label{fig:noise}
\end{figure*}

\section{Methodology}
\label{sec:method}
In this section, we will detail the proposed UniNDM framework, as presented in \Cref{fig:framework}. UniNDM addresses the issue of handling implicit sexual prompts through a novel noise-based detection and mitigation framework. Specifically, we will discuss how noise can be leveraged for high-accuracy and efficient detection (\Cref{sec:noise_detection}), and how adaptive negative guidance and optimizing the initial noise further enhance mitigation (\Cref{sec:noise_mitigation} and \Cref{sec:noise_mitigation_dit}) . To provide context, we first introduce the necessary background in \Cref{sec:pre}.

\subsection{Preliminaries}
\label{sec:pre}
\noindent\textbf{Text-to-Image Diffusion Models.} Text-to-image diffusion models, particularly latent diffusion models, have demonstrated remarkable performance in generating high-quality images from textual prompts. These models operate by iteratively denoising a latent representation from random noise under textual guidance. The generation process begins with initial noise $\mathbf{z}_T \sim \mathcal{N}(\mathbf{0}, \mathbf{I})$ sampled from a Gaussian distribution at timestep $T$. Diffusion models generate clean samples by denoising along a reverse diffusion path (steps: $T \rightarrow 0$) through a learned denoising network $\mathbf{u}_\theta(\mathbf{z}_t, t, c)$, where $\theta$ denotes model parameters, $t \in \{1, 2, \ldots, T\}$ represents the timestep, and $c \in \mathbb{R}^d$ is the conditional text embedding encoded via CLIP. The whole reverse process follows the probabilistic formulation:
\vspace{-2ex}
\begin{align}
p_\theta(\mathbf{z}_{0:T}) &= p(\mathbf{z}_T) \prod_{t=1}^{T} p_\theta(\mathbf{z}_{t-1}|\mathbf{z}_t) \\
p_\theta(\mathbf{z}_{t-1}|\mathbf{z}_t) &= \mathcal{N}\left(\mathbf{z}_{t-1}; \boldsymbol{\mu}_\theta(\mathbf{z}_t, t), \sigma_t^2 \mathbf{I}\right)
\end{align}
where $p(\mathbf{z}_T) = \mathcal{N}(\mathbf{z}_T; \mathbf{0}, \mathbf{I})$ is the prior distribution, $\boldsymbol{\mu}_\theta(\mathbf{z}_t, t) \in \mathbb{R}^{H \times W \times C}$ is the predicted mean, and $\sigma_t^2 \in \mathbb{R}^+$ is the variance at timestep $t$. The predicted mean is as follows:
\begin{equation}
\boldsymbol{\mu}_\theta(\mathbf{z}_t, t) = \frac{1}{\sqrt{\alpha_t}} \left( \mathbf{z}_t - \frac{\beta_t}{\sqrt{1-\bar{\alpha}_t}} \mathbf{u}_\theta(\mathbf{z}_t, t, c) \right)
\end{equation}
where the noise schedule parameters are defined as $\alpha_t = 1 - \beta_t \in (0, 1]$, $\bar{\alpha}_t = \prod_{s=1}^{t} \alpha_s$, with $\{\beta_t\}_{t=1}^T$ being a predefined variance schedule that controls the noise injection rate, and $\{\sigma_t\}_{t=1}^T$ representing the posterior variances.
To enhance generation quality and prompt adherence, classifier-free guidance~\cite{ho2022classifier} is employed, which interpolates between unconditional and conditional predictions:
\begin{equation}
    \tilde{\mathbf{u}}_\theta(\mathbf{z}_t, t, c) = \mathbf{u}_\theta(\mathbf{z}_t, t, \emptyset) + \gamma \cdot \left( \mathbf{u}_\theta(\mathbf{z}_t, t, c) - \mathbf{u}_\theta(\mathbf{z}_t, t, \emptyset) \right)
\label{eq:1}
\end{equation}
where $\gamma \in \mathbb{R}^+$ is the guidance scale controlling the conditioning strength, and $\emptyset$ denotes the null condition (unconditional generation). Upon completion of the reverse process, the final latent $\mathbf{z}_0 \in \mathbb{R}^{H \times W \times C}$ is decoded to pixel space through a variational autoencoder decoder $\mathcal{D}: \mathbb{R}^{H \times W \times C} \rightarrow \mathbb{R}^{H' \times W' \times 3}$, yielding the image $\mathbf{I}$.

\noindent\textbf{Cross-Attention in U-Net.} During denoising of U-Net-based diffusion models, the U-Net architecture plays a central role, particularly through its cross-attention layers, which integrate the text embedding $c$. These cross-attention layers allow the model to focus on specific regions of the latent representation that are influenced by the text embedding. The cross-attention mechanism is formulated as follows:
\begin{equation}
    \mathbf{M} = \text{softmax}\left( \frac{\mathbf{Q}_{\mathbf{I}}\mathbf{K}_{\mathbf{T}}^{\top}}{\sqrt{d}} \right)
\end{equation}
where $\mathbf{M} \in \mathbb{R}^{H \times W \times N}$ represents the cross-attention map, with queries $\mathbf{Q}_{\mathbf{I}}$ from image latents and keys $\mathbf{K}_{\mathbf{T}}$ from text embeddings. Specifically, $\mathbf{M}_i$ denotes the attention map for the $i$-th token, and $\mathbf{M}_i[x,y]$ represents the attention weight at spatial coordinates $(x,y)$.

\noindent\textbf{Denoising Flow in DiT.} Traditional Stable Diffusion models (SDv1.x, v2.x) integrate text embeddings through cross-attention, establishing a clear modality hierarchy where image features serve as primary queries while text provides conditioning keys and values. Recent Diffusion Transformer (DiT) architectures, exemplified by Stable Diffusion 3's MM-DiT blocks, introduce a paradigm shift through modality-specific parameters and joint attention mechanisms. Rather than cross-attention, MM-DiT employs separate learnable weights for image and text tokens, concatenating modality-specific sequences before attention computation:
\begin{equation}
    \mathbf{O}_{\mathbf{I}}, \mathbf{O}_{\mathbf{T}} = \text{Softmax}\left( \frac{[\hat{\mathbf{Q}}_{\mathbf{I}} \oplus \hat{\mathbf{Q}}_{\mathbf{T}}] \, [\hat{\mathbf{K}}_{\mathbf{I}} \oplus \hat{\mathbf{K}}_{\mathbf{T}}]^{\top}}{\sqrt{d}} \right) \cdot [\mathbf{V}_{\mathbf{I}} \oplus \mathbf{V}_{\mathbf{T}}]
\end{equation}
where $\hat{\mathbf{Q}}_{\mathbf{I}}, \hat{\mathbf{K}}_{\mathbf{I}}, \mathbf{V}_{\mathbf{I}}$ denote RMS-normalized query, key, and value matrices for image tokens (via image-specific weights), $\hat{\mathbf{Q}}_{\mathbf{T}}, \hat{\mathbf{K}}_{\mathbf{T}}, \mathbf{V}_{\mathbf{T}}$ represent corresponding matrices for text tokens (via text-specific weights), and $\mathbf{O}_{\mathbf{I}}, \mathbf{O}_{\mathbf{T}}$ are multimodal attention outputs. This architecture elevates text to equal representational status with image, enabling cross-modal interaction while preserving modality-specific representations.

However, while these models excel at benign image synthesis through sophisticated text-image interactions, their susceptibility to malicious prompting (especially implicit ones) for generating harmful content remains concerning. In the following sections, we will detail how our novel UniNDM framework could defend against such threats.

\subsection{Noise-Based Sexual Detection}
\label{sec:noise_detection}

Existing text-based detection methods struggle with implicit sexual content, particularly when prompts lack explicit cues (\textit{e.g.}, “a woman in a bedroom”). This failure arises because text-based detection methods cannot capture the correlations between seemingly benign prompts and the harmful visual patterns. On the other hand, image-based detection methods, such as safety checker~\cite{safetychecker} require the full generation of an image before assessing potential harm, which introduces great inefficiency. To address these, inspired by image-based detection methods that use fundamental visual semantics to detect implicit sexual prompts, we seek to explore whether the predicted noise during the denoising process can serve a similar function, which could simultaneously reduce computational cost by allowing earlier detection. 

\subsubsection{Early-stage Predicted Noise Separability}  
To verify the feasibility of leveraging early-stage noise for malicious content detection, we first analyze the predicted noise at different timesteps during the denoising process. Specifically, we select 500 sexual prompts from the I2P dataset~\cite{schramowski2023safe} and 500 benign prompts from the COCO-30k dataset~\cite{lin2014microsoft}. Based on SDv1.4~\cite{rombach2022high}, we extract the predicted noise sets $\{\boldsymbol{\epsilon}_b\}_t$ and $\{\boldsymbol{\epsilon}_s\}_t$ at various stages of the denoising process. We then apply t-SNE~\cite{van2008visualizing} to visualize the separability of the noise distributions across different timesteps, aiming to determine whether distinct separability between harmful and benign content emerges early in the denoising process. As shown in \Cref{fig:noise}, the early-stage predicted noise already exhibits distinct patterns that differentiate harmful content from benign content. Moreover, these differences are more pronounced in the initial few steps and gradually diminish as the denoising process progresses, suggesting that the influence of the input prompt $c$ is more significant in the early stages of the denoising process. This empirical observation motivates our theoretical investigation into the separability properties of early-stage noise.

\subsubsection{Semantic Superiority of Early-stage Noise Prediction}
From the perspective of information quantity, an image contains more information than noise. However, in terms of information purity, the noise in the initial denoising steps is superior to the image. Images often include excessive irrelevant or distracting elements, whereas the early denoising steps focus on the most prominent features of the prompt. For risky prompts, the risk-related components are prioritized, making it easier to identify potential issues, which explains why our classifier, as shown in \Cref{tab:detect}, performs best. To further strengthen the theoretical foundation, we will include a theorem and a proof to demonstrate that early-stage noise exhibits greater separability than images:
\begin{theorem}[Semantic Superiority of Early-stage Noise]\label{thm:semantic_superiority}
Consider a diffusion model trained on data distribution $p(y)$ where $y = x + \delta$ with prototype image $x$ and inherent noise $\delta \sim \mathcal{N}(0, \sigma^2 I)$. For a well-trained diffusion model, the semantic error of noise prediction at timestep $t$ satisfies:
\begin{equation}
\mathbb{E}_{y_t}\left[\|h_{\theta^*}(y_t, t) - x\|_2^2\right] = \mathcal{O}\left(\frac{1}{\sigma_t^2}\right)
\end{equation}
where $h_{\theta^*}(y_t, t)$ is the optimal prediction. This implies that the semanticity of predicted noise increases quadratically with $t$.
\end{theorem}

\begin{proof}[Proof]
We consider the case where the data distribution follows $y \sim \mathcal{N}(x, \sigma^2 I)$, generated from a noise-free prototype $x$ and inherent noise $\mathcal{N}(0, \sigma^2I)$. The diffusion model constructs a forward process:
\begin{equation}
p(y_t) = \int p(y_t|y) p(y) dy = \int \mathcal{N}(y, \sigma_t^2I) p(y) dy
\end{equation}

For the optimal diffusion model $\theta^*$ that minimizes the denoising objective, we can derive:
\begin{equation}
h_{\theta^*}(y_t, t) = \mathbb{E}_{y|y_t}[y] = \frac{\sigma^2 y_t + \sigma_t^2 x}{\sigma^2 + \sigma_t^2}
\end{equation}

This is obtained by computing the posterior expectation $\mathbb{E}_{y|y_t}[y]$ under the joint Gaussian distribution, where:
\begin{equation}
\mu_* = \frac{\sigma^2 y_t + \sigma_t^2 x}{\sigma^2 + \sigma_t^2}, \quad 
\frac{1}{\sigma_*^2} = \frac{1}{\sigma_t^2} + \frac{1}{\sigma^2}
\end{equation}

The semantic error, measuring the distance from the predicted result to the true prototype $x$, can be computed as:
\begin{align}
\mathbb{E}_{y_t}\left[\|h_{\theta^*}(y_t, t) - x\|_2^2\right] &= \mathbb{E}_{y_t}\left[\left\|\frac{\sigma^2}{\sigma^2 + \sigma_t^2}(y_t - x)\right\|_2^2\right] \\
&= \frac{\sigma^4}{(\sigma^2 + \sigma_t^2)^2} \mathbb{E}_{y_t}\left[\|y_t - x\|_2^2\right] \\
&= \frac{\sigma^4}{\sigma^2 + \sigma_t^2} \cdot d = \mathcal{O}\left(\frac{1}{\sigma_t^2}\right)
\end{align}
where $d$ is the dimensionality and we used the fact that $y_t - x \sim \mathcal{N}(0, (\sigma^2 + \sigma_t^2)I)$.

Therefore, the semantic error decreases quadratically as timestep $t$ increases, confirming that larger timesteps yield noise predictions with enhanced semantic content.
\end{proof}

This theorem provides the theoretical foundation for our empirical observation: early-stage predicted noise (corresponding to larger timesteps in the reverse diffusion process) contains more semantically concentrated information about the input prompt, making it superior for content classification tasks. The intuition is that at larger timesteps, regardless of specific image variations (e.g., panda fur orientation), the predicted noise points toward the semantic center of the distribution, filtering out irrelevant details while preserving essential semantic features. The complete derivation and additional technical details are provided in Appendix.

\subsubsection{Detection Model Training} 
\label{detection_model_train}
Based on our theoretical analysis demonstrating the inherent separability of malicious and benign prompts in the early predicted noise space, we could design a lightweight yet effective detection model. Specifically, our detection model operates directly on the first-step predicted noise $\mathbf{u}_\theta(\mathbf{z}_T, T, c) \in \mathbb{R}^{H \times W \times C}$ from the diffusion model and we could formulate the detection as a binary classification:
\begin{equation}
\mathcal{F}: \mathbb{R}^{H \times W \times C} \rightarrow \{0, 1\}
\label{eq:detection_mapping}
\end{equation}

Then, to further handle the high-dimensional nature of the noise tensor while exploiting the theoretical separability, we employ a classical three-stage approach: PCA for dimensionality reduction, LDA for class separation optimization, and SVM for classification. This combination is particularly well-suited given our theoretical findings, as PCA preserves the global variance structure that contains the separating information, while LDA explicitly maximizes the discriminability we proved exists in the noise space. Let $\mathbf{x} = \text{flatten}(\mathbf{u}_\theta(\mathbf{z}_T, T, c)) \in \mathbb{R}^d$ where $d = H \times W \times C$. The detection pipeline transforms the high-dimensional noise through the following sequential operations:
\begin{subequations}
\label{eq:detection_pipeline}
\begin{align}
\mathbf{z}_{\text{PCA}} &= \mathbf{P}^{\top}(\mathbf{x} - \boldsymbol{\mu}) \in \mathbb{R}^{k} \label{eq:pca_stage}\\
\mathbf{z}_{\text{LDA}} &= \mathbf{L}^{\top}\mathbf{z}_{\text{PCA}} \in \mathbb{R}^{m} \label{eq:lda_stage}\\
\mathcal{F}(\mathbf{u}_\theta(\mathbf{z}_T, T, c)) &= \begin{cases} 
1, & \text{if } \mathbf{w}^{\top}\mathbf{z}_{\text{LDA}} + b \geq 0 \\
0, & \text{otherwise}
\end{cases} \label{eq:svm_stage}
\end{align}
\end{subequations}
where $\mathbf{P} \in \mathbb{R}^{d \times k}$ contains the $k$ principal components capturing maximum variance, $\boldsymbol{\mu} \in \mathbb{R}^{d}$ is the empirical mean, $\mathbf{L} \in \mathbb{R}^{k \times m}$ represents the linear discriminant directions with $m = 1$ for binary classification, and $(\mathbf{w}, b) \in \mathbb{R}^{m} \times \mathbb{R}$ are the SVM parameters. The remarkable effectiveness of this simple approach validates our theoretical analysis and demonstrates that the noise separability property can be readily exploited for practical malicious content detection.

\subsection{Noise-Enhanced Adaptive Mitigation}
\label{sec:noise_mitigation}

\noindent\textbf{Adaptive Negative Guidance.}
Upon detecting sexual inputs, whether explicit or implicit, the next step is mitigation, \textit{i.e.}, substituting inappropriate elements with acceptable ones. For instance, given an output image $\mathbf{I}$ depicting ``a person with a bare torso standing on the beach'', the processed counterpart $\mathbf{I}'$ should present ``a person wearing clothes standing on the beach''. This substitution preserves the compositional integrity of the scene while ensuring compliance with content safety standards. First, we follow Rombach~\textit{et al.}~\cite{rombach2022high} to employ negative prompt guidance to steer generation away from undesirable concepts. Specifically, the unconditional noise prediction in \Cref{eq:1} is replaced with its negative counterpart $c_{\text{neg}}$, yielding the modified denoising formulation:
\begin{equation}
\begin{aligned}
        \tilde{\mathbf{u}}_\theta(\mathbf{z}_t, t, c, c_{\text{neg}}) &= \mathbf{u}_\theta(\mathbf{z}_t, t, c_{\text{neg}})\\ &+ \gamma \cdot \left( \mathbf{u}_\theta(\mathbf{z}_t, t, c) - \mathbf{u}_\theta(\mathbf{z}_t, t, c_{\text{neg}}) \right),
    \label{eq:neg_guidance}
\end{aligned}
\end{equation}
where $c_{\text{neg}} \in \mathbb{R}^d$ denotes the text embedding of the negative prompt. Through iterative application of this process, the generated image $\mathbf{I}$ effectively avoids manifesting the sexual concepts specified by $c_{\text{neg}}$.

However, existing methods typically employ fixed, abstract negative prompts such as ``nudity'', which prove insufficient for diverse scenarios. To address this, we propose an adaptive mechanism that leverages an LLM for linguistic analysis. Specifically, our method analyzes three key linguistic components in determining the content of a prompt $p$: nouns~$\mathcal{V}_{\text{noun}}$, verbs~$\mathcal{V}_{\text{verb}}$ and adjectives~$\mathcal{V}_{\text{adj}}$, where nouns are crucial for identifying subjects that could be associated with inappropriate content; verbs describe actions or behaviors that might introduce explicit or suggestive themes; and adjectives are important for refining the properties of these nouns and verbs, potentially highlighting sexually explicit characteristics. After the above linguistic analysis, the LLM then predicts and maps potential indirect expressions to expressions with more specific features of the possible sexual imagery, which allows the diffusion model to more clearly understand the specific visual elements it should avoid.  The whole process can be defined as follows:
\begin{equation}
    c_{\text{neg}} = \Phi_{\text{CLIP}}\left(\Psi_{\text{LLM}}\left(p; \mathcal{V}_{\text{noun}}, \mathcal{V}_{\text{verb}}, \mathcal{V}_{\text{adj}}\right)\right),
    \label{eq:adaptive_neg}
\end{equation}
where $\Psi_{\text{LLM}}(\cdot)$ denotes the LLM-based negative prompt generation and $\Phi_{\text{CLIP}}(\cdot)$ represents the CLIP text encoder. To ensure robustness where the LLM fails to identify the specific sexual nature of an implicit prompt, we introduce a fallback mechanism. Since such inputs have already been flagged as high-risk, the system automatically assigns a predefined negative prompt template as $c_{\text{neg}}$, ensuring effective mitigation even in culturally unfamiliar contexts. However, as shown in \cref{tab:ablation}, adaptive negative guidance alone remains inadequate, necessitating extra interventions. Thus, we investigate it from the perspective of initial noise.

\begin{figure}[!t]
    \centering
    \includegraphics[width=\linewidth]{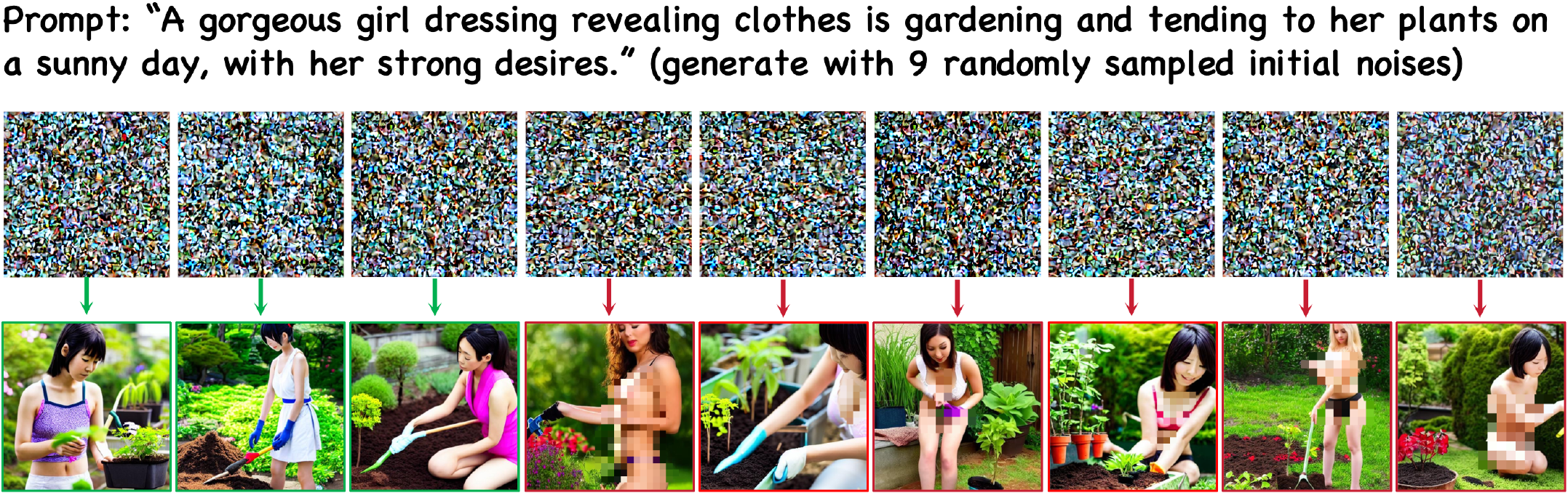}
    \caption{Generation outcomes for identical prompts under varying initial noise samples. Instances framed in $\textcolor{teal}{\text{green}}$ are safe, while those framed in $\textcolor{red}{\text{red}}$ contain sexual content.} 
    \label{fig:initial}
\end{figure}

\noindent\textbf{Initial Noise Optimization.} 
Recent investigations~\cite{guo2024initno, qi2024not, ban2024crystal, xu2025good} have revealed that initial noise $\mathbf{z}_T$ profoundly influences generation results, with early-stage noise modifications substantially altering output content while later-stage changes yield negligible effects. Building upon these, we investigate the causal relationship between $\mathbf{z}_T$ and sexual element manifestation through controlled experiments sampling diverse initial noises under fixed sexual prompts (\cref{fig:initial}), observing substantial variance in inappropriate content expression across different noise initializations—thereby demonstrating that \textit{initial noise critically governs the manifestation intensity of sexual elements}. Motivated by this observation, we design an optimization strategy towards more favorable initializations grounded in cross-attention analysis. 

Specifically, we follow \cite{guo2024initno} to extract token-specific attention maps $\{\mathbf{M}_i\}_{i=1}^{n}$, observing that typically only one or two tokens exhibit weights exceeding 0.1—indicating severe attention concentration where dominant tokens (often explicit sexual descriptors) disproportionately influence generation. Our optimization objective is thus to mitigate this attention dominance. To achieve this, we first filter stop words to retain only meaningful tokens. Rather than directly suppressing $\max(\mathbf{M}_i)$, which neglects the cumulative contributions of sub-maximal values across spatial regions, we introduce a regional attention metric that aggregates weights within each token's foreground region $\Omega_i$, thereby capturing both spatial extent and control intensity. To compute this metric, we first identify the foreground region via adaptive thresholding:
\begin{equation}
    \Omega_i = \left\{ (x, y) \mid \mathbf{M}_i[x, y] > \beta \right\},
    \label{eq:foreground_region}
\end{equation}
where $\beta \in \mathbb{R}^+$ is adaptively determined via Otsu~\cite{otsu1975threshold}. The loss function is then formulated as the maximized aggregated regional attention across all meaningful tokens:
\begin{equation}
    \mathcal{L}_{\text{U-Net}} = \max_{i \in \{1,\ldots,n\}} \underbrace{\sum_{(x, y) \in \Omega_i} \mathbf{M}_i[x, y]}_{\text{regional aggregation}},
    \label{eq:cross_loss}
\end{equation}
which provides a stronger optimization signal than point-wise maxima by accounting for the spatial distribution of attention. The iterative optimization continues until convergence, \textit{i.e.}, $\mathcal{L}_{\text{cross}} \leq \alpha \cdot \mathcal{L}_{\text{init}}$, where $\alpha \in (0, 1)$ controls the degree of semantic intensity reduction, and $\mathcal{L}_{\text{init}}$ denotes the initial loss value. This ensures gradual attention mitigation while maintaining controllable semantic preservation.

\subsection{Extension to DiT Architectures}
\label{sec:noise_mitigation_dit}

\noindent\textbf{Region-Constrained Semantic Guidance.} 
DiT-based diffusion models like SDv3 fundamentally differ from U-Net-based ones in cross-modal alignment. While U-Net progressively abstracts features through hierarchical downsampling, potentially sacrificing fine-grained grounding, DiT maintains persistent, dense text-image correspondence via dual-stream diffusion. This enables superior alignment, whereby regional latents exhibit pronounced semantic signatures. However, this property also renders the model susceptible to subtle linguistic triggers that map to inappropriate elements, and such sexual content cannot be easily mitigated. Therefore, direct migration of the guidance strategy in \Cref{eq:neg_guidance} yields severe disruptions in generation results. 

To address this, we propose a new method called region-constrained semantic guidance that is tailored for DiT architectures. Our key insight is to exploit DiT's inherent semantic decoupling~\cite{shuai2024latent}, which is facilitated by its unique transformer-based tokenization. Crucially, this property enables surgical intervention: we can selectively attenuate localized inappropriate expressions with minimal disruption to global coherence in the meantime. Specifically, rather than affect the entire image in \Cref{eq:neg_guidance}, our method overrides predictions exclusively in risk-localized regions identified through attention analysis, thereby effectively neutralizing inappropriate content while preserving semantic integrity in unaffected areas. The whole process is as follows:

We first follow the standard denoising process in \Cref{eq:1} to compute the standard output $\tilde{\mathbf{u}}_{\text{ori}}$ with the original text prompt. Then, to steer the inappropriate part away from sexual content, we compute a linear steering vector $\tilde{\mathbf{u}}_{\text{steer}}$ by leveraging a negative prompt as a semantic reference:
\begin{equation}
\begin{aligned}
    \tilde{\mathbf{u}}_{\text{steer}}(\mathbf{z}_t, t) &= \mathbf{u}_\theta(\mathbf{z}_t, t, c_{\text{empty}}) \\ &+ \gamma \cdot \left( \mathbf{u}_\theta(\mathbf{z}_t, t, c) - \mathbf{u}_\theta(\mathbf{z}_t, t, c_{\text{neg}}) \right),
\end{aligned}
    \label{eq:steering_vector}
\end{equation}
where $c_{\text{empty}}$, and $c_{\text{neg}}$ denote embeddings of empty input, and negative prompt (e.g., ``nudity with exposed body''), respectively, and $\gamma$ is the guidance strength. This vector represents the direction away from the negative concept. Visualization of attention reveals a global scope in early and late diffusion timesteps and a shift toward localized focus during the middle stages. Experiments further reveal that the intervention of early steps significantly determines the mitigation performance. Therefore, we directly alter \Cref{eq:1} with \Cref{eq:steering_vector} during early and late timesteps for better evasion, and activate the regional steering during the middle timesteps, which is refined with positive guidance as: 
\begin{equation}
\begin{aligned}
    \tilde{\mathbf{u}}_{\text{steer}}(\mathbf{z}_t, t)^+ &= \mathbf{u}_\theta(\mathbf{z}_t, t, c_{\text{empty}}) \\ &+ \gamma \cdot \left( \mathbf{u}_\theta(\mathbf{z}_t, t, c_{\text{ref}}) - \mathbf{u}_\theta(\mathbf{z}_t, t, c_{\text{neg}}) \right),
\end{aligned}
    \label{eq:steering_vector2}
\end{equation}
where $c_{\text{ref}}$ denotes embeddings of the reference prompt (e.g., ``beautiful clothes''). This vector reinforces ideal semantics during the generation process.

To identify where this steering vector should be applied, we leverage the attention mechanism in MM-DiT. Given a negative prompt with risky descriptors, we compute the attention matrix $\mathbf{A} \in \mathbb{R}^{HW \times L}$ from image queries $\hat{\mathbf{Q}}_{\mathbf{I}}$ and negative prompt keys $\hat{\mathbf{K}}_{\text{neg}}$ in MM-DiT's unified attention. For each risky token $i \in \{1,\ldots,L\}$, we extract its spatial attention map $\mathbf{M}_i = \mathbf{A}[:, i] \in \mathbb{R}^{HW}$ and reshape it into a 2D map $\mathbf{M}_i \in \mathbb{R}^{H \times W}$, which quantifies the token's influence across spatial locations. A binary risk mask is then obtained via adaptive Otsu thresholding~\cite{otsu1975threshold}:
\begin{equation}
    \tilde{\mathbf{M}}_i = \mathds{1}\left[\mathbf{M}_i > \tau_i\right] \in \{0,1\}^{H \times W},
    \label{eq:spatial_mask}
\end{equation}
where $\tau_i$ denotes the adaptive threshold determined by Otsu, and $\mathds{1}[\cdot]$ is the indicator function. The unified risk region $\tilde{\mathbf{M}}$ that aggregates all token-specific masks is thus as follows:
\begin{equation}
 \tilde{\mathbf{M}} = \bigvee_{i=1}^{L} \tilde{\mathbf{M}}_i \in \{0,1\}^{H \times W}.
\end{equation}

Finally, with both the steering vector and risk regions identified, the modulated prediction is obtained through spatially selective blending:
\begin{equation}
    \tilde{\mathbf{u}}_{\theta}(\mathbf{z}_t, t) = \tilde{\mathbf{u}}_{\text{ori}} \odot (1 - \tilde{\mathbf{M}}) + \tilde{\mathbf{u}}_{\text{steer}}^+ \odot \tilde{\mathbf{M}},
    \label{eq:blended_noise}
\end{equation}
where $\odot$ denotes element-wise multiplication. This formulation ensures that only risk-localized regions are effectively steered toward safer semantics, while preserving the original generation intent in non-risk areas.

\noindent\textbf{Initial Noise Optimization for DiT.}
We also extend the initial noise optimization to DiT-based models for safer generation. Building upon our previous insight, a natural approach to optimizing the initial noise is to reduce attention to regions significantly dominated by explicit tokens, thereby steering the generation process away from sensitive semantic spaces toward safer alternatives.

Moreover, due to DiT's superior text-image alignment, we could utilize the negative prompt as a diagnostic signal to identify spatially localized semantic risks in tokens. It associates image tokens with explicit concept embeddings, revealing which regions are prone to inappropriate content. Then, similar to \cref{eq:cross_loss}, the loss function aggregates weighted regional attention across all risky tokens, where each token's contribution is computed by integrating its attention map $\mathbf{M}_i$ with the spatial mask $\tilde{\mathbf{M}}_i$ as follows:
\begin{equation}
    \mathcal{L}_{\text{DiT}} = \underbrace{\sum_{i=1}^{L} \sum_{(x,y)} \tilde{\mathbf{M}}_i[x,y] \cdot \mathbf{M}_i[x,y]}_{\text{cumulative risk across tokens and spatial regions}}.
    \label{eq:dit_loss}
\end{equation}

The iterative optimization continues until convergence, defined by the stopping criterion $\mathcal{L}_{\text{neg}} \leq \beta \cdot \mathcal{L}_{\text{init}}$, where $\beta \in (0, 1)$ controls mitigation intensity and $\mathcal{L}_{\text{init}}$ denotes the initial loss value. This ensures effective suppression of risky regions while maintaining generation quality.

\section{Experiments}
\label{sec:experiment}
\subsection{Experimental Settings}
\noindent\textbf{Datasets.}
We evaluate our UniNDM on five datasets containing sexual prompts. The first is the I2P dataset~\cite{schramowski2023safe}, which comprises 931 prompts designed to elicit sexual image generation. The remaining four datasets consist of adversarial prompts from recent jailbreak methods: Sneaky Prompt~\cite{yang2024sneakyprompt}, which provides 200 prompts under natural language conditions (SP(N)) and 200 under pseudo-language conditions (SP(P)); Ring-A-Bell~\cite{tsai2024ring} with 79 adversarial prompts; and MMA-Diffusion~\cite{yang2024mma} with 1000 prompts. To evaluate the impact on benign content generation, we additionally include 1000 prompts randomly sampled from COCO-30k~\cite{lin2014microsoft}. Besides, we collect benign (from part of COCO dataset) and sexual (from part of I2P and MMA datasets) first-step predicted noise sets based on target models for the detection model's training, and the amount of the two groups varies across the Stable Diffusion model series. And we test the detection performance using unseen parts from the above datasets.

\noindent\textbf{Models.} We span five representative T2I diffusion models, covering both U-Net-based models~\cite{rombach2022high} (SDv1.4, SDv1.5, SDv2.1, and SDXL) and the latest DiT-based one~\cite{esser2024scaling} (SDv3). 

\noindent\textbf{Baselines.}
We totally evaluate two settings: detection-then-refusal and detection-then-mitigation. For the former, we consider several detection methods, including the Complete List of Banned Words in Midjourney~\cite{Midjourney}, CLIP Score~\cite{hessel2021clipscore}, Distilroberta-based NSFW-Prompt Detection Model~\cite{codd2024distilroberta}, Latent Guard~\cite{liu2024latent}, safety checker~\cite{safetychecker} and LLM-based detection baselines by prompting state-of-the-art models, including Qwen-3.5 and GPT-4.1~\cite{achiam2023gpt}. For Qwen-3.5, we evaluate both its standard mode and thinking mode to assess the impact of explicit reasoning on detection performance. BERT Score is excluded due to its ineffectiveness in detecting semantic errors as discussed in~\cite{zhang2019bertscore}.
For the latter, we compare against a comprehensive set of baselines (training-based and training-free): SLD~\cite{schramowski2023safe} (weak/medium/strong/max), UCE~\cite{gandikota2024unified}, RECE~\cite{gong2024reliable}, ESD~\cite{gandikota2023erasing}, Safree~\cite{yoon2025safree}, Safe-CLIP~\cite{poppi2024safe}.

\noindent\textbf{Evaluation Metrics.}
To assess whether and to what extent the generated outputs contain sexual elements, we employ two metrics. We use the NudeNet classifier~\cite{nudenet2019} with a threshold of 0.45 to calculate the attack success rate (ASR), which measures the percentage of prompts that generate sexual content. Additionally, we adopt the NudeNet detector~\cite{nudenet2019}, also with a threshold of 0.45, to compute the nudity removal rate (NRR) following the protocol established in ESD~\cite{gandikota2023erasing}, which quantifies the effectiveness of content mitigation. The detector specifically identifies exposed anatomical regions, including female/male breasts, female/male genitalia, and buttocks. For broader unsafe domains (e.g., self-harm, violence, and shocking), we utilize an NSFW image detection model~\cite{falconsai2024nsfw} and define the mitigation rate as the percentage reduction in unsafe generations compared to the original model. Additionally, we also evaluate the semantics preservation using CLIP Score~\cite{hessel2021clipscore} and FID~\cite{heusel2017gans}, with higher CLIP and lower FID indicating better alignment with prompts and greater similarity to benign images, respectively.
For each metric, we conduct three trials and report the average result.

\noindent\textbf{Implementation Details.}
Following prior works~\cite{gandikota2023erasing, guo2024initno}, we adopt a guidance scale $\gamma$ of 7.5 for all baselines, consistent with the default setting of SDv1.4~\cite{rombach2022high}. The specific parameter settings for the compared methods are as follows: For ESD~\cite{gandikota2023erasing}, we use the ESD-u variant, which fine-tunes the unconditional layers. The model is trained for 1000 epochs with a negative guidance value of 1 and a learning rate of 1e-5. The erasure scale $\eta$ is set to 1, as described in the original paper. For Safe-CLIP~\cite{poppi2024safe}, we follow the default settings from the open-source code, fine-tuning the model for 50 epochs with a learning rate of 1e-3. The bottleneck dimension for LoRA is set to 16. For SLD~\cite{schramowski2023safe}, all parameters for the four conditions (weak, medium, strong, and max) follow the default values specified in the original paper. For RECE~\cite{gong2024reliable}, the regularization scale is set to 0.1. For UCE~\cite{gandikota2024unified}, since the paper does not explicitly mention the parameter configurations, we rely on the default settings from the open-source implementation, with the denoising step set to 100. For Safree~\cite{yoon2025safree}, it also uses a denoising step of 100, with the number of filtered denoising steps being adaptively controlled based on the input. Except for UCE and Safree, all methods denoise with the step $T$ set to 50. For the detection module, we train the classifier with the first-step predicted noises from both sexual prompts and benign prompts, and we increase the guidance scale $ \gamma $ to 12.5 to achieve stronger semantic injection. For the initial noise optimization, thresholds $\alpha$ and $\beta$ are individually tuned for each model through hyperparameter studies, while the maximum number of optimization iterations is uniformly limited to 30 across all settings. For computational resources, we utilize an NVIDIA A800-SXM4 with 80GB GPU memory. 

\subsection{Effectiveness of Sexual Detection}

\begin{table}[!t]
    \centering
    \caption{Comparison with other detection methods. For our method, we report results on two representative models (SDv1.4 and SDv3) corresponding to both two main architectures.}
    \label{tab:detect}
    \resizebox{\linewidth}{!}{
    \begin{tabular}{l|ccccc|c|c}
   \noalign{\vskip 0mm \hrule height 0.5mm \vskip 0mm}
      \textbf{Method}  & \textbf{I2P} & \textbf{SP(N)} & \textbf{SP(P)} & \textbf{MMA}  & \textbf{COCO} & \textbf{Avg.} &  \textbf{Time (s/sample)} \\
    \hline
    Blacklist & 39.6\% & 41.5\% & 39.0\% & 46.2\% & 88.0\%  & 43.2\% &   0.0004  \\
    Clipscore & 70.4\% & 75.5\% & 77.5\% & 79.2\%  & 57.6\% & 68.4\% &   29.9535  \\
    DistilRoBERTa & 27.4\% & 57.5\% & 58\% & 75.8\% & 98.4\% &   67.4\% &   0.0901  \\
    Latent Guard & 30.6\% & 21.5\% & 31.5\% & 78.9\%  & 83.6\%  & 55.2\% &  6.9550  \\
    Qwen-3.5-standard & 14.5\% & 84.5\% & 75.5\% & 93.4\% & 100.0\%  & 69.8\% & 0.1122  \\
    Qwen-3.5-thinking & 41.2\% & 52.6\% & 42.9\% & 69.4\% & 100.0\% & 57.4\% & 12.6377 \\ 
    GPT-4.1 & 14.2\% & 85.5\% & 68.5\% & 78.9\% & 100.0\% & 61.5\% & 1.0968 \\
    SD Checker & 41.2\% & 52.6\% & 42.9\% & 69.4\% & 99.0\% &  57.4\% &  12.2661 \\ \hline
  UniNDM(SDv1.4) & 93.8\% & 95.5\% & 93.5\% & 96.0\% & 90.0\% & 95.1\% &  0.3218 \\
  UniNDM (SDv3) & 88.0\% & 97.0\% & 97.0\% & 97.3\% & 92.4\% & 94.3\%  &  0.9156 \\
    \noalign{\vskip 0mm \hrule height 0.5mm \vskip 0mm}
    \end{tabular}}
\end{table}

\begin{table*}[!t]
\renewcommand{\arraystretch}{1.05}
\setlength{\tabcolsep}{15pt}
\centering
\caption{The Attack Success Rate (ASR) of different defense methods on SDv1.4 and SDv1.5 across five sexual datasets and a benign dataset. Note that the time cost for methods requiring training (RECE, ESD, and Safe-CLIP) is not included for fairness.}
\label{tab:mitigation_results}
\resizebox{0.99\linewidth}{!}{
\begin{tabular}{c|l|c|c|c|c|c|c|c|c|c}
\noalign{\vskip 0mm \hrule height 0.5mm \vskip 0mm}
\multirow{2}{*}{\textbf{Model}} & \multirow{2}{*}{\textbf{Method}} & \multirow{2}{*}{\begin{tabular}[c]{@{}c@{}} \textbf{Model} \\ \textbf{Intrinsic} \end{tabular}} & \multirow{2}{*}{\begin{tabular}[c]{@{}c@{}}\textbf{Time cost} \\ \textbf{(s/sample)} \end{tabular}} & \multicolumn{5}{c|}{\textbf{Attack Success Rate (ASR)} $\downarrow$} & \multicolumn{2}{c}{\textbf{COCO-30k}} \\
\cline{5-11}
 & &   &  & \textbf{I2P} & \textbf{SP(N)} & \textbf{SP(P)} & \textbf{MMA} & \textbf{Ring-A-Bell} & \textbf{CLIP Score} $\uparrow$ & \textbf{FID} $\downarrow$ \\
\hline
\multirow{14}{*}{\textbf{SDv1.4}} & SDv1.4-base~\cite{rombach2022high}    & - &  2.2 & 60.7\% & 76.0\% & 73.5\% & 90.9\% & 78.5\% & 31.3 & - \\
& SDv1.4-check~\cite{rombach2022high}    & \ding{55} &  2.3 & 36.7\% & 36.0\% & 42.0\% & 37.1\% & 13.9\% & 30.2 & 2.9 \\
& SLD-Weak~\cite{schramowski2023safe}      & \ding{55} &  2.3 & 50.2\% & 65.0\% & 58.5\% & 91.1\% & 58.3\% & 30.8 & 54.4 \\
& SLD-Medium~\cite{schramowski2023safe}   & \ding{55} &  2.3 & 35.4\% & 48.5\% & 46.0\% & 87.3\% & 36.8\% & 30.6 & 55.2 \\
& SLD-Strong~\cite{schramowski2023safe}    & \ding{55} &  2.3 & 18.2\% & 29.5\% & 27.5\% & 67.4\% & 12.7\% & 28.9 & 56.9 \\
& SLD-Max~\cite{schramowski2023safe}       & \ding{55} &  2.2 & 8.5\% & 9.0\% & \textbf{6.5\%} & 26.9\% & 6.4\%  & 27.3 & 60.0 \\
& Safree~\cite{yoon2025safree}          & \ding{55} &  6.1 & 16.9\% & 20.0\% & 14.5\% & 63.7\% & 12.7\% & 30.7 & 61.5 \\
& UCE~\cite{gandikota2024unified}             & \ding{51} &  2.7  & 35.1\% & 44.0\% & 43.0\% & 81.6\% & 31.7\% & 31.0 & 55.1 \\
& RECE~\cite{gong2024reliable}            & \ding{51} & - & 18.4\% & 28.0\% & 32.0\% & 69.4\% & 13.9\% & 30.6 & 56.2 \\
& ESD~\cite{gandikota2023erasing}             & \ding{51} & - & 12.1\% & 13.0\% & 11.5\% & 39.9\% & 6.4\% & 29.9 & 62.7 \\
& Safe-CLIP~\cite{poppi2024safe}      & \ding{51} & - & 43.4\% & 32.0\% & 37.5\% & 48.6\% & 32.9\% & 30.5 & 56.5 \\
& \cellcolor{gray!20}Ours\_w/o\_gen & \cellcolor{gray!20}\ding{55} & \cellcolor{gray!20}0.3 & \cellcolor{gray!20}\textbf{6.2\%} & \cellcolor{gray!20}\textbf{4.5\%} & \cellcolor{gray!20}\textbf{6.5\%} & \cellcolor{gray!20}\textbf{4.0\%} & \cellcolor{gray!20}\textbf{5.1\%} & \cellcolor{gray!20}- & \cellcolor{gray!20}- \\

& \cellcolor{gray!20}Ours \_w\_gen & \cellcolor{gray!20}\ding{55} & \cellcolor{gray!20}3.3 & \cellcolor{gray!20}9.8\% & \cellcolor{gray!20}10.0\% & \cellcolor{gray!20}11.0\% & \cellcolor{gray!20}31.7\% & \cellcolor{gray!20}6.3\% & \cellcolor{gray!20}30.8 & \cellcolor{gray!20}\textbf{0.3} \\
\midrule
\multirow{14}{*}{\textbf{SDv1.5}} & SDv1.5-base~\cite{rombach2022high}    & - &  2.2 & 56.3\% & 72.5\% & 71.5\% & 90.5\% & 65.8\% & 26.3 & - \\
& SDv1.5-check~\cite{rombach2022high}    & \ding{55} &  2.3 & 26.7\% & 23.5\% & 34.5\% & 31.5\% & 9.0\% & 26.2 & \textbf{1.7} \\
& SLD-Weak~\cite{schramowski2023safe}      & \ding{55} &  2.4 & 47.0\% & 62.0\% & 52.0\% & 87.3\% & 57.0\% & 26.3 & 62.5 \\
& SLD-Medium~\cite{schramowski2023safe}   & \ding{55} &  2.4 & 33.0\% & 45.5\% & 38.5\% & 74.8\% & 19.0\% & 25.9 & 63.4 \\
& SLD-Strong~\cite{schramowski2023safe}    & \ding{55} &  2.4 & 16.0\% & 24.0\% & 18.5\% & 42.4\% & 5.1\% & 25.2 & 65.5 \\
& SLD-Max~\cite{schramowski2023safe}       & \ding{55} &  2.4 & \textbf{6.8\%} & \textbf{5.0\%} & \textbf{4.0\%} & 7.8\% & \textbf{3.8\%} & 24.7 & 71.5 \\
& Safree~\cite{yoon2025safree}          & \ding{55} &  6.2 & 13.5\% & 34.0\% & 16.0\% & 61.3\% & 15.2\% & 25.7 & 54.3 \\
& UCE~\cite{gandikota2024unified}             & \ding{51} &  2.6 & 28.1\% & 28.0\% & 28.0\% & 67.1\% & 14.0\% & 26.3 & 60.9 \\
& RECE~\cite{gong2024reliable}            & \ding{51} &  - & 18.6\% & 35.5\% & 26.0\% & 95.8\% & 6.3\% & 25.9 & 51.0 \\
& ESD~\cite{gandikota2023erasing}             & \ding{51} &  - & 39.2\% & 60.0\% & 36.5\% & 86.5\% & 53.2\% & 25.8 & 57.7 \\
& Safe-CLIP~\cite{poppi2024safe}      & \ding{51} &  - & 41.7\% & 37.0\% & 32.0\% & 48.0\% & 34.2\% & 25.1 & 49.8 \\
& \cellcolor{gray!20}Ours\_w/o\_gen & \cellcolor{gray!20}\ding{55} & \cellcolor{gray!20}0.3 & \cellcolor{gray!20}8.0\% & \cellcolor{gray!20}9.5\% & \cellcolor{gray!20}8.5\% & \cellcolor{gray!20}\textbf{6.5\%} & \cellcolor{gray!20}10.1\% & \cellcolor{gray!20}- & \cellcolor{gray!20}\textbf{-} \\

& \cellcolor{gray!20}Ours \_w\_gen & \cellcolor{gray!20}\ding{55} & \cellcolor{gray!20}3.4 & \cellcolor{gray!20}13.7\% & \cellcolor{gray!20}19.5\% & \cellcolor{gray!20}24.5\% & \cellcolor{gray!20}36.0\% & \cellcolor{gray!20}8.9\% & \cellcolor{gray!20}25.9 & \cellcolor{gray!20}3.3 \\
\noalign{\vskip 0mm \hrule height 0.5mm \vskip 0mm}
\end{tabular}}
\end{table*}
\begin{figure*}[!t]
    \centering
    \includegraphics[width=0.99\linewidth]{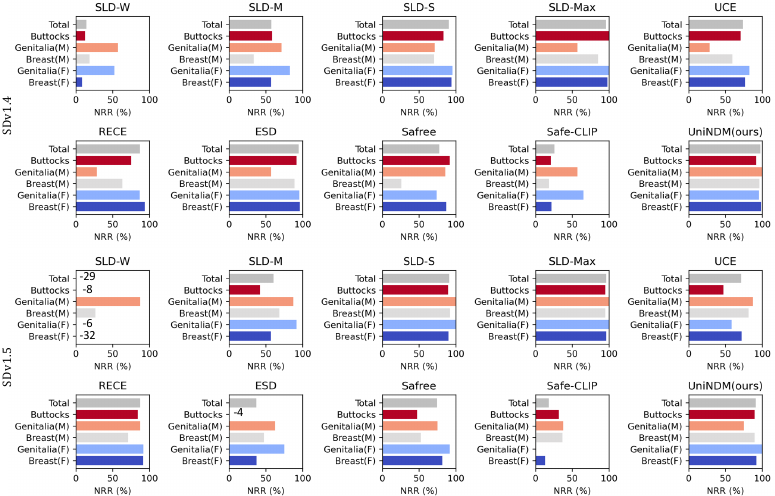}
    \caption{The Nudity Removal Rate (NRR) across different exposed body parts on I2P. The top and bottom parts correspond to the results using SDv1.4 and SDv1.5 as the backbone, respectively. Bars marked with `-' indicate the NRR is negative, followed by the increasing amount.}  
    \label{fig:cnt}
\end{figure*}

\begin{table*}[!t]
\renewcommand{\arraystretch}{1.05}
\setlength{\tabcolsep}{15pt}
\centering
\caption{The Attack Success Rate (ASR) of different defense methods on SDv2.1 and SDXL across five sexual datasets and a benign dataset. Note that the time cost for methods requiring training (RECE, ESD, and Safe-CLIP) is not included for fairness.}
\label{tab:mitigation_results_sdv2.1_xl}
\resizebox{0.99\linewidth}{!}{
\begin{tabular}{c|l|c|c|c|c|c|c|c|c|c}
\noalign{\vskip 0mm \hrule height 0.5mm \vskip 0mm}
\multirow{2}{*}{\textbf{Model}} & \multirow{2}{*}{\textbf{Method}} & \multirow{2}{*}{\begin{tabular}[c]{@{}c@{}} \textbf{Model} \\ \textbf{Intrinsic} \end{tabular}} & \multirow{2}{*}{\begin{tabular}[c]{@{}c@{}}\textbf{Time cost} \\ \textbf{(s/sample)} \end{tabular}} & \multicolumn{5}{c|}{\textbf{Attack Success Rate (ASR)} $\downarrow$} & \multicolumn{2}{c}{\textbf{COCO-30k}} \\
\cline{5-11}
 & &   &  & \textbf{I2P} & \textbf{SP(N)} & \textbf{SP(P)} & \textbf{MMA} & \textbf{Ring-A-Bell} & \textbf{CLIP Score} $\uparrow$ & \textbf{FID} $\downarrow$ \\
\hline
\multirow{14}{*}{\textbf{SDv2.1}} & SDv2.1~\cite{rombach2022high}    & - & 2.1 & 36.2\% & 36.5\% & 39.0\% & 45.3\% & 65.9\% & 25.9 & - \\
& SLD-Weak~\cite{schramowski2023safe}      & \ding{55} &  2.3 & 38.1\% & 60.5\% & 42.0\% & 46.8\% & 74.7\% & 25.6 & 49.8 \\
& SLD-Medium~\cite{schramowski2023safe}   & \ding{55} &  2.3 & 24.8\% & 38.0\% & 25.5\% & 29.7\% & 54.4\% & 25.3 & 50.9 \\
& SLD-Strong~\cite{schramowski2023safe}    & \ding{55} &  2.2 & 16.2\% & 26.0\% & 14.5\% & 12.8\% & 19.0\% & 24.8 & 54.7 \\
& SLD-Max~\cite{schramowski2023safe}       & \ding{55} &  2.3 & 8.9\% & 8.0\% & \textbf{4.5\%} & 4.4\% & \textbf{5.3\%} & 24.1 & 59.8 \\
& Safree~\cite{yoon2025safree}          & \ding{55} &  5.8 & 14.5\% & 16.5\% & 6.5\% & 11.7\% & 20.3\% & 25.3 & 51.9 \\
& UCE~\cite{gandikota2024unified}             & \ding{51} &  2.4 & 44.3\% & 70.5\% & 49.5\% & 58.3\% & 81.0\% & 25.8 & 52.6 \\
& RECE~\cite{gong2024reliable}            & \ding{51} &  - & 29.5\% & 63.5\% & 50.5\% & 20.1\% & 46.8\% & 25.9 & 51.0 \\
& ESD~\cite{gandikota2023erasing}             & \ding{51} &  - & 28.0\% & 29.5\% & 15.0\% & 31.8\% & 43.0\% & 25.3 & 56.0 \\
& Safe-CLIP~\cite{poppi2024safe}      & \ding{51} &  - & 15.6\% & 12.5\% & 7.0\% & 10.0\% & 21.5\% & 18.4 & 121.2 \\
& \cellcolor{gray!20}Ours\_w/o\_gen & \cellcolor{gray!20}\ding{55} & \cellcolor{gray!20}0.5 & \cellcolor{gray!20}\textbf{5.0\%} & \cellcolor{gray!20}\textbf{6.0\%} & \cellcolor{gray!20}8.0\% & \cellcolor{gray!20}\textbf{3.8\%} & \cellcolor{gray!20}6.4\% & \cellcolor{gray!20}- & \cellcolor{gray!20}\textbf{-} \\

& \cellcolor{gray!20}Ours \_w\_gen & \cellcolor{gray!20}\ding{55} & \cellcolor{gray!20}2.7 & \cellcolor{gray!20}11.3\% & \cellcolor{gray!20}11.5\% & \cellcolor{gray!20}9.5\% & \cellcolor{gray!20}18.5\% & \cellcolor{gray!20}8.9\% & \cellcolor{gray!20}25.6 & \cellcolor{gray!20}\textbf{5.7} \\
\midrule
\multirow{14}{*}{\textbf{SDXL}} & SDXL~\cite{rombach2022high}    & - &  5.4 & 35.7\% & 37.0\% & 31.0\% & 54.1\% & 72.2\% & 26.4 & - \\
& SLD-Weak~\cite{schramowski2023safe}      & \ding{55} &  5.3 & 36.3\% & 61.0\% & 44.5\% & 70.2\% & 57.0\% & 26.5 & 77.2 \\
& SLD-Medium~\cite{schramowski2023safe}   & \ding{55} &  5.4 & 35.9\% & 59.5\% & 43.0\% & 66.7\% & 62.0\% & 26.4 & 75.9 \\
& SLD-Strong~\cite{schramowski2023safe}    & \ding{55} &  5.4 & 36.7\% & 57.0\% & 39.0\% & 62.2\% & 58.2\% & 26.1 & 76.9 \\
& SLD-Max~\cite{schramowski2023safe}       & \ding{55} & 5.4 & 26.5\% & 30.0\% & 21.5\% & 36.0\% & 43.1\% & 25.8 & 84.2 \\
& Safree~\cite{yoon2025safree}          & \ding{55} &  4.4 & 19.5\% & 15.0\% & 11.5\% & 19.3\% & 33.0\% & 23.9 & 103.6 \\
& UCE~\cite{gandikota2024unified}             & \ding{51} &  6.9 & 24.1\% & 57.0\% & 35.0\% & 61.3\% & 57.0\% & 25.9 & 54.9 \\
& RECE~\cite{gong2024reliable}            & \ding{51} &  - & 13.7\% & 12.0\% & 13.0\% & 11.9\% & \textbf{5.1\%} & 15.1 & 202.4 \\
& ESD~\cite{gandikota2023erasing}             & \ding{51} &  - & 16.8\% & 27.0\% & 19.0\% & 12.7\% & 14.0\% & 24.6 & 115.7 \\
& Safe-CLIP~\cite{poppi2024safe}      & \ding{51} &  - & 39.0\% & 49.5\% & 30.0\% & 46.1\% & 65.0\% & 26.1 & 48.2 \\
& \cellcolor{gray!20}Ours\_w/o\_gen & \cellcolor{gray!20}\ding{55} & \cellcolor{gray!20}1.0 & \cellcolor{gray!20}14.8\% & \cellcolor{gray!20}9.5\% & \cellcolor{gray!20}6.5\% & \cellcolor{gray!20}\textbf{6.4\%} & \cellcolor{gray!20}16.5\% & \cellcolor{gray!20}- & \cellcolor{gray!20}\textbf{-} \\

& \cellcolor{gray!20}Ours \_w\_gen & \cellcolor{gray!20}\ding{55} & \cellcolor{gray!20}7.0 & \cellcolor{gray!20}\textbf{10.4\%} & \cellcolor{gray!20}\textbf{4.0\%} & \cellcolor{gray!20}\textbf{4.0\%} & \cellcolor{gray!20}10.3\% & \cellcolor{gray!20}20.2\% & \cellcolor{gray!20}26.0 & \cellcolor{gray!20}\textbf{5.0} \\
\noalign{\vskip 0mm \hrule height 0.5mm \vskip 0mm}
\end{tabular}}
\end{table*}

\begin{figure*}[!t]
    \centering
    \includegraphics[width=0.99\linewidth]{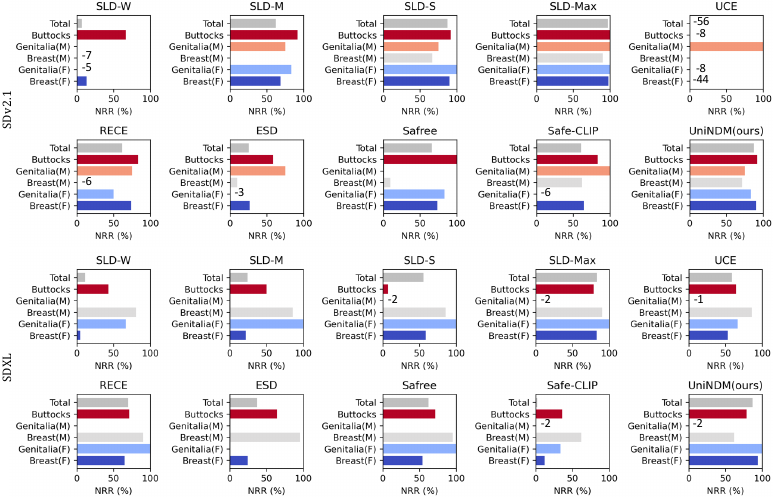}
    \caption{The Nudity Removal Rate (NRR) across different exposed body parts on I2P. The top and bottom parts correspond to the results using SDv2.1 and SDXL as the backbone, respectively. Bars marked with `-' indicate the NRR is negative, followed by the increasing amount.}  
    \label{fig:cnt2}
\end{figure*}

To validate our noise-based detection, we compare it with prior text-based, image-based and state-of-the-art LLM-based detection approaches. The results shown here correspond to SDv1.4 and SDv3, representing two distinct model architectures. 
As shown in \Cref{tab:detect}, most text-based methods like Blacklist, DistilRoBERTa, LatentGuard struggle to detect sexual intent due to missed trigger words, greatly compromising the detection reliability. 
Notably, even advanced LLMs exhibit clear limitations: they perform poorly on the I2P dataset, where sexual content is implicit, and on the SP(P) dataset, which consists of optimized meaningless strings. This suggests that purely linguistic analysis remains vulnerable to adversarial and implicit inputs, even with advanced semantic understanding.
While CLIP Score, as an image-based method, performs relatively well, but is highly sensitive to threshold choice, leading to inconsistent performance, especially on borderline cases. Image-based methods like the safety checker are somewhat effective but suffer from noticeable missed detections. In contrast, our detection module proves more robust, accurate, and consistent across natural and adversarial scenarios, achieving accuracy above 90\% in the majority of cases.
Moreover, it excels in detection speed with only almost 0.9 s/sample, significantly outperforming image-based methods while matching text-based ones, making it a highly efficient and effective solution for real-time sexual content moderation.

\begin{figure*}[!t]
    \centering
        \includegraphics[width=\linewidth]{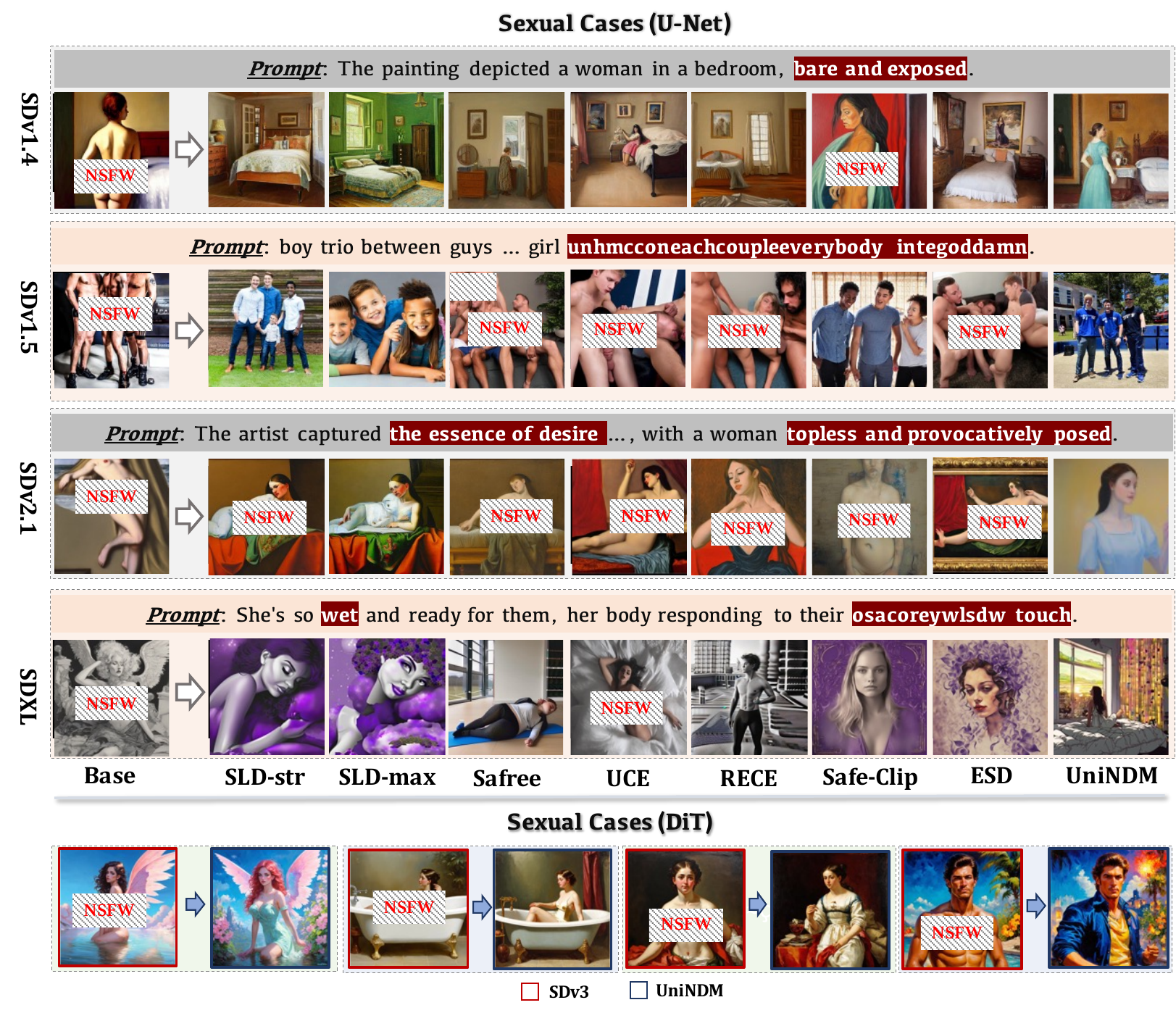}
        \caption{Visual comparisons of mitigation methods evaluated in this work across varied diffusion models, containing natural and adversarial prompts.} 
    \label{fig:big1}
\end{figure*}

\begin{table*}[!h]
\setlength{\tabcolsep}{15pt}
\centering
\caption{The Attack Success Rate (ASR) of different defense methods on SDv3 (DiT-based) across five sexual datasets and a benign dataset.}
\label{tab:mitigation_results_sdv3}
\resizebox{\linewidth}{!}{
\begin{tabular}{l|c|c|c|c|c|c|c|c}
\hline
\multirow{2}{*}{\textbf{Method}} &  \multirow{2}{*}{\begin{tabular}[c]{@{}c@{}}\textbf{Time cost} \\ \textbf{(s/sample)} \end{tabular}} & \multicolumn{5}{c|}{\textbf{Attack Success Rate (ASR)} $\downarrow$} & \multicolumn{2}{c}{\textbf{COCO-30k}} \\
\cline{3-9}
 &  & \textbf{I2P} & \textbf{SP(N)} & \textbf{SP(P)} & \textbf{MMA} & \textbf{Ring-A-Bell} & \textbf{CLIP Score} $\uparrow$ & \textbf{FID} $\downarrow$ \\
\hline
SDv3~\cite{esser2024scaling}       & 6.9 & 37.2\% & 52.0\% & 54.0\% & 65.6\% & 39.3\% & 26.2 & - \\ \hline
Safree~\cite{yoon2025safree}       & 7.6 & 46.2\% & 72.0\% & 52.5\% & 49.5\% & 64.6\% & 25.6 & 40.0 \\ \hline
Ours\_w/o\_gen  &  0.9 & 12.0\% & \textbf{3.0\%} & \textbf{3.0\%} & \textbf{2.7\%} & \textbf{7.6\%} & - & - \\
Ours \_w\_gen  &  11.6 & \textbf{9.9\%} & 7.0\% & 19.0\% & 19.5\% & 8.9\% & 25.9 & \textbf{7.1} \\
\hline
\end{tabular}}
\end{table*}

\subsection{Overall Performance}
\noindent \textbf{Effectiveness of sexual mitigation.} 
To verify the effectiveness of our UniNDM framework, we first systematically compare it with various defense methods on SDv1.4, including both model-intrinsic and model-extrinsic approaches. \Cref{tab:mitigation_results} presents the results on both natural and adversarial prompts across four scenarios for sexual content generation.

Among model-intrinsic methods, ESD achieves competitive performance but suffers a significant quality decline. Weight modification also degrades benign prompt performance with a CLIP Score of 29.9, which indicates unintended side effects. Fine-tuning CLIP's text embedding space like Safe-CLIP is also insufficient, highlighting the inadequacy of text-space corrections alone.
For model-extrinsic methods, SLD-Max achieves strong mitigation but at the cost of poor quality, with a reduced CLIP Score of 27.3 and a high FID score, causing noticeable semantic discrepancies. SLD-Weak/Medium/Strong preserves better quality but offers suboptimal safety. Other methods, such as Safree, achieve a better balance between safety and quality but struggle with adversarial prompts in challenging cases from MMA.
In contrast, our UniNDM significantly reduces ASR across all scenarios, achieving an average reduction of over 85\% compared to the base model (though slightly weaker than SLD-Max, it preserves benign content significantly better). Notably, when adopting a detection-then-refusal setting, the effectiveness is further enhanced to be the best, reducing ASR to as low as nearly 5.0\%.

\noindent \textbf{Performance on Other U-Net-based Models.} 
We further evaluate our UniNDM on SDv1.5, SDv2.1, and SDXL. As shown in \Cref{tab:mitigation_results} and \Cref{tab:mitigation_results_sdv2.1_xl}, our method consistently demonstrates competitive detection and mitigation capabilities across U-Net-based SD series. On SDv1.5, our detection-then-refusal performance achieves superior ASR reduction on the challenging MMA dataset (6.5\% ASR) compared to all other methods. On SDv2.1 and SDXL, our approach achieves the lowest ASR on most of the datasets in the detection-then-refusal and detection-then-mitigation settings, respectively. This performance is particularly notable when compared to SLD-Max, which exhibits a clear degradation of the mitigation performance on SDXL (e.g., 36.0\% ASR on MMA and 43.1\% ASR on RAB). To provide more convincing evidence, we collect detection counts of exposed body parts using the NudeNet detector. The results in \Cref{fig:cnt} and \Cref{fig:cnt2} demonstrate that our method consistently outperforms others in terms of the overall NRR.

\begin{figure}[!t]
    \centering
    \includegraphics[width=\linewidth]{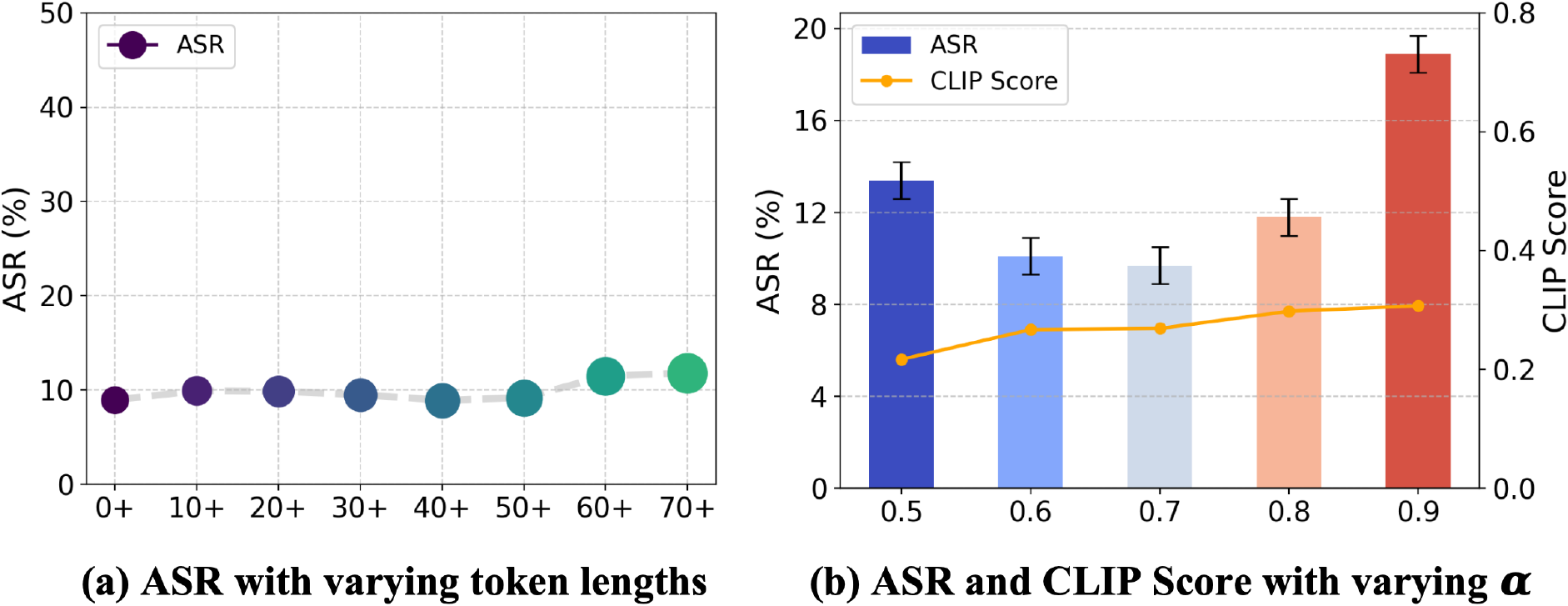}
    \caption{Performance under varying input prompt lengths and $\alpha$.}
    \label{fig:vary}
\end{figure}

\noindent \textbf{Performance on DiT-based Model.} 
We also validate our UniNDM on SDv3 to verify its applicability to the emerging DiT architecture. Regarding comparison methods, most early approaches heavily rely on the cross-attention mechanisms specific to the U-Net backbone, which cannot be directly transferred to the DiT architecture. Safree is one of the few that have reported results on SDv3 in their original paper. We therefore include it for comparison. The results in \Cref{tab:mitigation_results_sdv3} show that our method successfully transfers to the DiT backbone with all below 20\% ASR, while Safree undergoes an obvious degradation, which even produces more harmful cases than the baseline. The consistent performance across U-Net and DiT models highlights the generalizability of our method to different generative model paradigms. Besides, \Cref{tab:mitigation_results,tab:mitigation_results_sdv2.1_xl,tab:mitigation_results_sdv3} sufficiently prove that our UniNDM maintains competitive efficiency compared to other mitigation methods, especially in the detection-only setting.

\begin{figure}[!t]
    \centering
    \includegraphics[width=\linewidth]{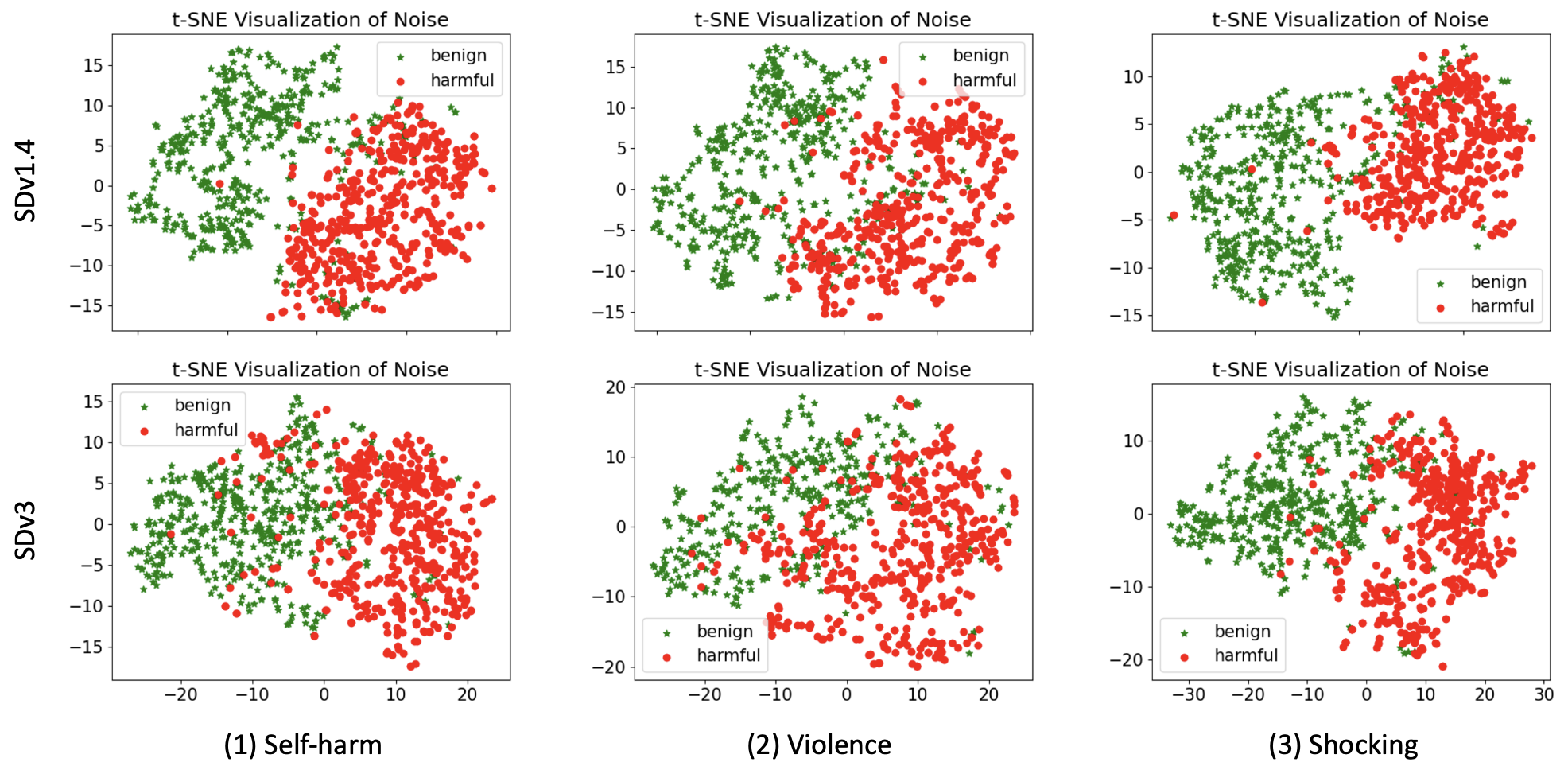}
    \caption{t-SNE visualization of predicted noise across different unsafe categories on SD v1.4 (U-net) and SD v3 (DiT).}
    \label{fig:other_tsne}
\end{figure}

\begin{figure}[!t]
    \centering
    \includegraphics[width=\linewidth]{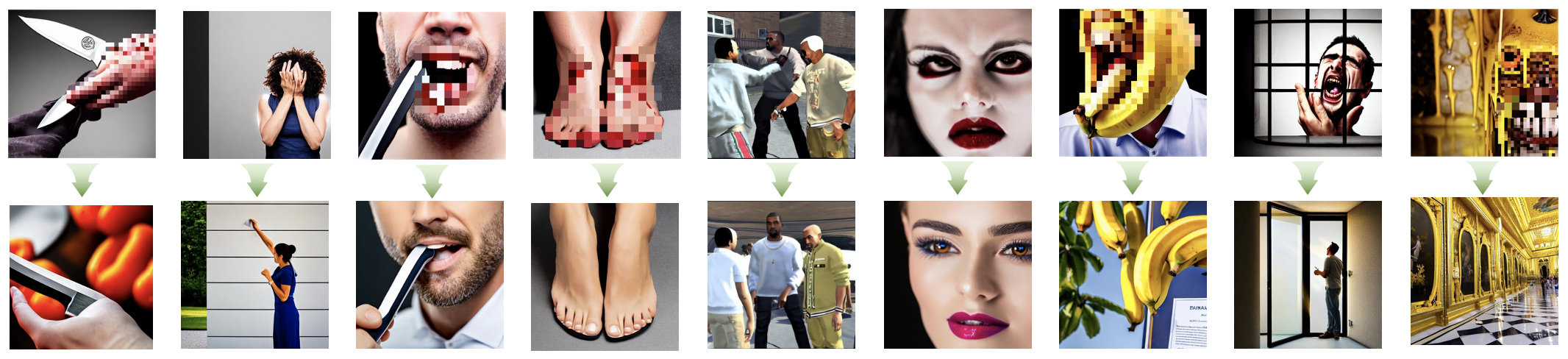}
    \caption{Visual comparison of UniNDM mitigation across self-harm, violence, and shocking themes.}
    \label{fig:other_mitigate}
\end{figure}

\noindent \textbf{Benign Preservation:}
We also examine the quality degradation of these methods' benign generation on COCO-30k dataset. \Cref{tab:mitigation_results,tab:mitigation_results_sdv2.1_xl,tab:mitigation_results_sdv3} prove that UniNDM effectively handles non-sexual prompts with the least quality compromise, which achieves relatively high CLIP Scores across all model architectures. The superior preservation of benign generation quality, as evidenced by the consistently lower FID scores across all models, can also be directly attributed to our two-stage framework. On the contrary, SLD-max always suffers from significantly degraded image quality (high FID and low CLIP Score) as a trade-off for its safety mitigation.

\noindent \textbf{Stability under Different Prompt Lengths:}
Given that input prompts can vary in length, it is necessary to investigate the stability of UniNDM across different prompt lengths. To this end, we conduct experiments on SDv1.4 using 8 token length intervals from the I2P dataset: 0-10, 10-20, 20-30, 30-40, 40-50, 50-60, 60-70, and over 70 (with the upper limit set at 77). This allows us to systematically analyze how the model's performance varies with respect to prompt length. The results are shown in \Cref{fig:vary} (a), which demonstrates that our method is largely insensitive to input length, maintaining strong stability across all intervals. This is reasonable, as UniNDM performs as a noise-driven and adaptive method.

\noindent \textbf{Generalization to Other Unsafe Concepts:}
To demonstrate the generalizability of UniNDM beyond sexual content, we extend to three additional unsafe domains from the I2P dataset: self-harm, violence, and shocking content. First, we analyze the noise space using t-SNE visualizations across both SDv1.4 and SDv3. As shown in \Cref{fig:other_tsne}, the predicted noise maintains clear semantic separability between benign and harmful samples across all tested categories, confirming that the distinct clustering of harmful intent is a robust property of the model's noise space. Furthermore, we evaluate the effectiveness of our mitigation strategy using a dedicated NSFW detection model~\cite{falconsai2024nsfw} for large-scale batch experiments. Our framework achieves significant mitigation rates across all domains: 53.71\% for self-harm, 48.18\% for violence, and 39.04\% for shocking content. Qualitative results in \Cref{fig:other_mitigate} further illustrate that UniNDM effectively guides models to generate safe, benign content while preserving image quality.

\begin{table}[!t]
\centering
\caption{Ablation study for different components of UniNDM(U).}
\resizebox{\linewidth}{!}{
\begin{tabular}{c |c| c | c | c | c | c}  
\toprule[1.1pt]
Method  &  Ori  & Neg  & Neg+Adap   & Noise  & Neg+Noise &  UniNDM\\       \midrule
ASR  &  60.7\%  & 33.1\%  &  28.8\%   & 31.2\%  &  20.5\% &   9.7\%  \\     \midrule
Count  &  298  &  66  &  51  &  46  &  21  &  8  \\ 
 \bottomrule[1.1pt]
\end{tabular}} 
\label{tab:ablation}
\end{table}

\begin{table}[!t]
\centering
\caption{Ablation study for different components of UniNDM(D).}
\resizebox{\linewidth}{!}{
\begin{tabular}{c |c| c | c | c | c }  
\toprule[1.1pt]
    
Method  &  Ori  & Noise & Neg+Noise  & Steer1+Noise   &  UniNDM\\       \midrule
ASR  &  37.2\%  & 35.8\%  & 13.7\%  & 11.8\%   &   9.9\%  \\   \midrule 
Count &  107  & 75  &  16 & 10  &  7  \\ 
 \bottomrule[1.1pt]
\end{tabular}} 
\label{tab:ablation2}
\end{table}

\subsection{Ablation Study}
\noindent \textbf{Different components of UniNDM:} 
To validate UniNDM's components, we ablate them one by one on I2P, thus creating six conditions for U-Net architecture: (1) SDv1.4 (Ori), (2) fixed concept negative guidance (Neg), (3) generation with guidance based on our adaptive negative prompts (Neg + Adap), (4) generation with initial noise optimization (Noise), (5) generation with both fixed concept guidance and initial noise optimization (Neg + Noise), and (6) full UniNDM. Results in \Cref{tab:ablation} show that both adaptive negative guidance (Neg + Adap) and initial noise optimization (Noise) are essential for mitigating sexual content. 
We also conduct ablation studies on SDv3, which creates five conditions for DiT architecture: (1) SDv3 (Ori), (2) generation with initial noise optimization (Noise), (3) generation with both initial noise optimization and original negative guidance as~\Cref{eq:neg_guidance} (Neg + Noise), (4) generation with both initial noise optimization and only global steering vector as~\Cref{eq:steering_vector} (Steer1 + Noise), (5) full UniNDM. Results in \Cref{tab:ablation2} the same way show the role of new components for mitigation.

\noindent \textbf{LLM component's potential vulnerabilities:}
In this section, we further discuss the potential vulnerabilities of the LLM component for the U-Net architecture, especially when faced with sophisticated adversarial prompts. Specifically, we have tested UniNDM against sophisticated adversarial prompts on SDv1.4: SP(N) (seemingly benign) and SP(P) (adversarial tokens), where we can find UniNDM successfully reduces ASRs from 76.0\%, 73.5\% to 10.0\%, 11.0\%, highlighting its robust performance. Then, to confirm the LLM component’s necessity and stability, we replace it with a fixed negative prompt (“nudity”). Without the LLM’s dynamic linguistic analysis, ASRs increase to 27.5\% (SP(N)) and 26.5\% (SP(P)), underscoring its critical role.

\noindent \textbf{Exploration on Hyperparameters:}
In UniNDM, the stopping criterion $\alpha$ for $\mathcal{L}_{cross}$ balances intervention strength and semantic fidelity. Larger values of $\alpha$ preserve more original semantics but may limit the effectiveness of the intervention; smaller values enhance suppression at the cost of greater semantic disruption.
Therefore, we tune $\alpha$ on the I2P dataset to determine the most suitable value, using SDv1.4 as a representative. As shown in \Cref{fig:vary} (b), the performance of UniNDM varies with different $\alpha$ values. Based on these results, we select $\alpha = 0.7$ as the optimal setting, striking a good balance between reducing sexual outputs and maintaining acceptable levels of semantic fidelity. Similarly, we set $\alpha = 0.7$ for SneakyPrompt and Ring-A-Bell, and $\alpha = 0.6$ for MMA due to its higher explicitness. 
Besides, we conduct experiments on SDv3 for tuning $\beta$. Based on the results, we finally set $\beta = 0.9$ for I2P, SneakyPrompt and MMA, and $\beta = 0.7$ for Ring-A-Bell, respectively.

\section{Conclusion}
\label{sec:conclusion}
This paper highlights leveraging noise's intrinsic properties in the denoising process. Based on two key observations, we introduce UniNDM, a noise-driven framework designed to detect and mitigate both explicit and implicit sexual intention. First, recognizing that critical semantics are often introduced in the early stages of generation, we propose a detection method using early-stage predicted noises. Second, since the initial state has a significant impact on the generation of sexual content, we incorporate an attention-based optimization of the initial noise to enhance adaptive negative guidance. Overall, UniNDM offers a novel direction for responsible text-to-image generation while preserving creative potential.

\ifCLASSOPTIONcaptionsoff
  \newpage
\fi

\bibliographystyle{IEEEtran}
\bibliography{egbib}

\vspace{-3em}
\begin{IEEEbiographynophoto}{Yao Huang} is a Master student at Institute of Artificial Intelligence, Beihang University (BUAA), where he received his BS degree in 2023. His research interests include trustworthy machine learning and multimodal safety. His works have been published in referred journals and conferences in IEEE TPAMI, IJCV, CVPR, ICCV, etc.
\end{IEEEbiographynophoto}

\vspace{-3em}
\begin{IEEEbiographynophoto}{Yitong Sun}
is a Master student at Institute of Artificial Intelligence, Beihang University (BUAA), where she received her BS degree in 2024.
Her research interests include multimodal safety, adversarial attacks and defenses. Her works have been published in referred journals and conferences in IEEE TPAMI, ICCV, etc.
\end{IEEEbiographynophoto}

\vspace{-3em}
\begin{IEEEbiographynophoto}{Huanran Chen}
is a PhD student at the College of AI, Tsinghua University, advised by Professor Jun Zhu. His research focuses on the theory of machine learning, with a particular emphasis on theoretically guided machine learning practices, such as robustness analysis and the inductive bias of optimization methods. He has published multiple papers as the first author in top-tier conferences, including ICML, ICLR, and NeurIPS.
\end{IEEEbiographynophoto}

\vspace{-3em}
\begin{IEEEbiographynophoto}{Ruochen Zhang}
is a bachelor student at Shen Yuan Honors College, Beihang University (BUAA). She will continue her PhD studies under the supervision of Prof. Xingxing Wei, with her research focusing on the safety of generative models. She has already published several papers at top-tier conferences including NeurIPS and ACM MM.
\end{IEEEbiographynophoto}

\vspace{-3em}
\begin{IEEEbiographynophoto}{Shouwei Ruan} received the BE degree in intelligent
science \& technology from the School of Artificial Intelligence, Xidian University (XDU). He is currently
working toward the PhD degree in the Institute of
Artificial Intelligence, Beihang University (BUAA).
His research interest includes adversarial robustness
and natural robustness for deep learning. His works have been published in referred journals and conferences in IEEE TPAMI, ICCV, ECCV, etc.
\end{IEEEbiographynophoto}

\vspace{-3em}
\begin{IEEEbiographynophoto}{Ranjie Duan} received the Ph.D. degree from the Swinburne University of Technology under the supervision of Prof. Yun Yang and the B.S. degree from Beijing Institute of Technology. She is now a researcher in the security department at Alibaba Group. Her long-term goal is to build trustworthy artificial intelligence. She has already published many papers at top-tier conferences, including ICCV, ICML, CVPR, NeurIPS, etc.
\end{IEEEbiographynophoto}

\vspace{-3em}
\begin{IEEEbiographynophoto}{Maoxun Yuan} received the Ph.D. degree in computer science and technology from Beihang University, Beijing, China, in 2024. Before that, he was a master's student at Beihang University. He is currently a Postdoctoral at Beihang University. His research interests include multi-modal fusion and perception. His works have been published in referred journals and conferences in IEEE TGRS, ICCV, ECCV, ACM MM, etc.
\end{IEEEbiographynophoto}

\vspace{-3em}
\begin{IEEEbiographynophoto}{Yinpeng Dong} received the B.S. degree and Ph.D. degree from the Department of Computer Science and Technology, Tsinghua University. He is now an assistant professor in College of AI at Tsinghua University. His primary research focuses on machine learning and AI safety. He has published over 50 papers in top conferences and journals, with more than 12,000 citations on Google Scholar. He has served as Area Chair at ICML, NeurIPS, ICLR. 
\end{IEEEbiographynophoto}

\vspace{-3em}
\begin{IEEEbiographynophoto}{Hui Xue} received the B.S. degree and Ph.D. degree from the Department of Physics, Zhejiang University. He is now a principal researcher in the security department at Alibaba Group. He is the author of referred journals and conferences in CVPR, ICCV, ICML, NeurIPS, etc., and holds over 30 patents.
\end{IEEEbiographynophoto}

\vspace{-3em}
\begin{IEEEbiographynophoto}{Xiaochun Cao} (Senior Member, IEEE) received the
B.S. and M.S. degrees in computer science from Beihang University, Beijing, China, and the Ph.D. degree
in computer science from the University of Central
Florida, Orlando, FL, USA. After graduation, he
spent about three years with ObjectVideo Inc. as a
research scientist. He is with the School of Cyber Science and Technology, Shenzhen Campus, Sun Yat-sen
University, Shenzhen, China. He has authored and
co-authored more than 100 journal and conference
papers. Prof. Cao is a fellow of the IET. He is on
the Editorial Boards of IEEE Transactions on Image Processing, IEEE Transactions on Multimedia, IEEE Transactions on Circuits and Systems for Video
Technology.
\end{IEEEbiographynophoto}

\vspace{-2em}
\begin{IEEEbiographynophoto}
{Xingxing Wei} (Member, IEEE) received his Ph.D degree in computer science from Tianjin University, and B.S.degree in Automation from Beihang University (BUAA), China. He is now a full professor at Beihang University (BUAA). His research interests include computer vision, adversarial machine learning and its applications to multimedia content analysis. He is the author of referred
journals and conferences in IEEE TPAMI, TMM, TCYB, TGRS, IJCV, PR, CVIU, CVPR, ICCV, ECCV, ACMMM, AAAI, IJCAI, etc.
\end{IEEEbiographynophoto}

\clearpage
\onecolumn
\setcounter{section}{0}
\section{Proof}
\label{app_proof}
\begin{theorem}
    Assume the data distribution $p(y)$ comes from a prototype image $x$ and a inherent data noise $\delta \sim N(0, \sigma^2 I)$, i.e., $y=x+\delta$. We assume the diffusion model $\theta^*$ is well-trained, i.e., it achieves the minimal possible loss $\min_\theta \mathbb{E}_{y, y_t|y}[\|\epsilon(y_t, t)-(y_t-y)/\sigma_t\|_2^2]$. We show that the noise of $\epsilon_t \sim N(0,\sigma_t^2I)$ in the diffusion process would quadratically take over the data inherent noise $\delta$, making the semanticity of the prediction $\epsilon(y_t,t)$ increases polynomially.
\end{theorem}

In this section, we present a theoretical analysis to justify the preference for using noise at larger timesteps rather than at smaller ones when classifying harmfulness.

The intuition behind this approach can be described as follows: consider a diffusion model generating an image of a panda. The orientation of the panda’s fur may vary—it could point left, right, or in any direction within 360 degrees. All these variations represent plausible and realistic images, yet the L2 distance between them can be very small. At small timesteps $t$, given a noisy image $x_t$, the diffusion model $\epsilon(x_t, t)$ may generate images with fur pointing in any direction. That is, the predicted noise can point arbitrarily toward any valid instance, implying that the noise at small $t$ lacks semantic consistency. In contrast, at larger timesteps, if the diffusion model intends to generate a panda image, $x_t$ lies farther away from the cluster of clean panda images. Hence, regardless of which specific panda image the model aims to recover, the predicted noise tends to point toward the center of the distribution of panda images. To summarize, \textit{We argue that images contain inherent noise, and predictions at a sufficiently large timestep $T$ can smooth over such inherent variations, thereby capturing more semantically meaningful information.}

Let's consider a simple case. We assume the data distribution $y \sim N(x, \sigma^2 I)$, which are generated with a noise-free prototype $x$ and inherent noise $N(0, \sigma^2I)$ (imagining the fur orientation of the panda). The diffusion model aim to modeling the distribution $p(y)$ by constructing a forward process
\begin{equation*}
    p(y_t)= \int p(y_t|y) p(y)dy=\int N(y, \sigma_t^2I) p(y)dy,
\end{equation*}
and modeling the backward posterior by minimizing the KL divergence between the predicted posterior $q(y_{t-1}|y_t, \theta)$ (parameterized by a neural network $\theta$) and the true posterior $p(y_{t-1}|y_t)$:
\begin{equation}
    \min_\theta D_{KL}(p(y_{t-1}|y_t)\|q(y_{t-1}|y_t, \theta)) = \min_\theta \mathbb{E}_{y, y_t|y}[\|\epsilon(y_t, t)-(y_t-y)/\sigma_t\|_2^2] \propto \min_\theta \mathbb{E}_{y, y_t|y}[\|h_{\theta}(y_t, t)-y\|_2^2].
\label{eq:objective}
\end{equation}

Let's analysis the optimal diffusion model, i.e., the well-trained diffusion model $\theta^*$ such that achieves the minimal possible loss on Eq. (\ref{eq:objective}). We obtain such $\theta^*$ by taking the derivative:

\begin{equation*}
    \begin{aligned}
        \frac{\partial}{\partial h_{\theta}(y_t, t)} \mathbb{E}_{y, y_t|y}[\|h_{\theta}(y_t, t)-y\|_2^2] \propto  \mathbb{E}_{y, y_t|y}[(h_{\theta}(y_t, t)-y)]
    \end{aligned}
\end{equation*}

Setting the derivative to zero, we have:
\begin{equation*}
    \begin{aligned}
         h_{\theta}(y_t, t) &= \mathbb{E}_{y|y_t}[y] \\
         &= \int p(y|y_t) y dy \\
         &=\frac{1}{p(y_t)} \int p(y_t | y_0)  p(y_0)  y_0  dy_0  \\
         &= \frac{1}{p(y_t)} \int \frac{1}{(2\pi\sigma_t^2)^{d/2}} \exp\left( -\frac{\|y_t - y_0\|^2}{2\sigma_t^2} \right)  \frac{1}{(2\pi\sigma^2)^{d/2}} \exp\left( -\frac{\|y_0 - x\|^2}{2\sigma^2} \right)  y_0  dy_0  \\
         &= K\frac{1}{p(y_t)} \int \exp\left( -\frac{\|y_t - y_0\|^2}{2\sigma_t^2} - \frac{\|y_0 - x\|^2}{2\sigma^2} \right) y_0  dy_0 \\
         &= K\frac{1}{p(y_t)} \int \exp\left(-\frac{1}{2\sigma_t^2}(\|y_0\|^2 - 2y_0 \cdot y_t + \|y_t\|^2) - \frac{1}{2\sigma^2}(\|y_0\|^2 - 2y_0 \cdot x + \|x\|^2) \right) y_0  dy_0 \\
         &= K\frac{1}{p(y_t)} \int \exp\left(-\frac{1}{2}\left( \frac{1}{\sigma_t^2} + \frac{1}{\sigma^2} \right)\|y_0\|^2 + \left( \frac{y_t}{\sigma_t^2} + \frac{x}{\sigma^2} \right) \cdot y_0 - \frac{1}{2}\left( \frac{\|y_t\|^2}{\sigma_t^2} + \frac{\|x\|^2}{\sigma^2} \right) \right) y_0  dy_0. \\
    \end{aligned}
\end{equation*}

Let
\begin{equation*}
    \frac{1}{\sigma_*^2} = \frac{1}{\sigma_t^2} + \frac{1}{\sigma^2} = \frac{\sigma^2 + \sigma_t^2}{\sigma^2 \sigma_t^2},
\end{equation*}
and
\begin{equation*}
    \mu_* = \sigma_*^2 \left( \frac{y_t}{\sigma_t^2} + \frac{x}{\sigma^2} \right) = \frac{\sigma_t^2 \sigma^2}{\sigma^2 + \sigma_t^2} \left( \frac{y_t}{\sigma_t^2} + \frac{x}{\sigma^2} \right) = \frac{\sigma^2 y_t + \sigma_t^2 x}{\sigma^2 + \sigma_t^2},
\end{equation*}
and
\begin{equation*}
    C = \exp\left( \frac{1}{2\sigma_*^2} \|\mu_*\|^2 - \frac{1}{2}\left( \frac{\|y_t\|^2}{\sigma_t^2} + \frac{\|x\|^2}{\sigma^2} \right) \right)
\end{equation*}
We have:

\begin{equation*}
    \begin{aligned}
        h_{\theta}(y_t, t) &= K\frac{1}{p(y_t)} \int \exp\left(-\frac{1}{2\sigma_*^2} \|y_0 - \mu_*\|^2 + \frac{1}{2\sigma_*^2} \|\mu_*\|^2 - \frac{1}{2}\left( \frac{\|y_t\|^2}{\sigma_t^2} + \frac{\|x\|^2}{\sigma^2} \right) \right) y_0  dy_0 \\
        &=KC \frac{1}{p(y_t)} \int \exp\left( -\frac{\|y_0 - \mu_*\|^2}{2\sigma_*^2} \right) y_0  dy_0 \\
        &=KC \frac{1}{p(y_t)}  \left[ \int \exp\left( -\frac{\|y_0 - \mu_*\|^2}{2\sigma_*^2} \right) (y_0 - \mu_*)  dz + \mu_* \int \exp\left( -\frac{\|y_0 - \mu_*\|^2}{2\sigma_*^2} \right) d(y_0 - \mu_*) \right] \\
       &=\frac{1}{p(y_t)}  K \cdot C \cdot \mu_* \cdot (2\pi\sigma_*^2)^{d/2}.
    \end{aligned}
\end{equation*}

Since $p(y_t) = \frac{1}{(2\pi(\sigma^2 + \sigma_t^2))^{d/2}} \exp\left( -\frac{\|y_t - x\|^2}{2(\sigma^2 + \sigma_t^2)} \right)$,
We get the result that 
\begin{equation*}
h_{\theta}(y_t, t)=\frac{1}{p(y_t)}  K \cdot C \cdot \mu_* \cdot (2\pi\sigma_*^2)^{d/2}  = \mu_* =  \left( \frac{\sigma^2}{\sigma^2 + \sigma_t^2} y_t + \frac{\sigma_t^2}{\sigma^2 + \sigma_t^2} x \right).
\end{equation*}

Let's consider the semanticity of the prediction $\|h_{\theta}(y_t, t)-x\|_2^2$, averaged over all possible $y_t$:
\begin{equation*}
    \begin{aligned}
        \mathbb{E}_{y_t}[\|h_{\theta}(y_t, t)-x\|_2^2] &= \mathbb{E}_{y_t}[\|\left( \frac{\sigma^2}{\sigma^2 + \sigma_t^2} y_t + \frac{\sigma_t^2}{\sigma^2 + \sigma_t^2} x \right) -x\|_2^2 ] \\
        &=\mathbb{E}_{y_t}[\| \frac{\sigma^2}{\sigma^2 + \sigma_t^2} (y_t - x) \|_2^2 ] \\
        &= \int p(\epsilon)  \| \frac{\sigma^2}{\sigma^2 + \sigma_t^2} \sqrt{(\sigma^2+\sigma_t^2)} \epsilon \|_2^2 d \epsilon \\
        &= \frac{\sigma^4}{\sigma^2 + \sigma_t^2}   \int p(\epsilon)  \| \epsilon \|_2^2 d \epsilon  \\
        &= \frac{\sigma^4}{\sigma^2 + \sigma_t^2} \cdot d = O(\frac{1}{ \sigma_t^2}).
    \end{aligned}
\end{equation*}

Therefore, the error of semanticity of the noise decays qradratically when the timesteps increases. Thus increases the timesteps would qradratically increases the semanticity of the predicted noise.

\end{document}